\documentclass{article}
\usepackage{kr}
\usepackage{times}
\usepackage{soul}
\usepackage{url}
\usepackage[hidelinks]{hyperref}
\usepackage[utf8]{inputenc}
\usepackage[small]{caption}
\usepackage{graphicx}
\usepackage{amsmath}
\usepackage{amsthm}
\usepackage{booktabs}
\usepackage{algorithm}
\usepackage{algorithmic}
\usepackage{amsfonts}
\usepackage{paralist}
\usepackage{todonotes}
\usepackage{tikz}
\usetikzlibrary{shapes,decorations,shadows,arrows,calc}
\makeatletter

\tikzstyle{arg}=[draw,circle,fill=gray!15,inner sep=1pt,minimum size=.45cm]
\tikzstyle{argd}=[draw,circle,gray!70,inner sep=1pt,minimum size=.45cm,dashed]
\tikzstyle{argb}=[draw,circle,fill=gray!15,inner sep=1pt,minimum size=.65cm]

\tikzstyle{argTD}=[draw, thick, circle, fill=gray!15,inner sep=0pt,minimum size=0.6cm,font=\small]
\tikzstyle{argR}=[draw, thick, circle, fill=gray!15,inner sep=0pt,minimum size=0.45cm,font=\small]
\tikzstyle{scc}=[draw, thick, rectangle,align=center, fill=gray!15,inner sep=0pt,minimum size=0.8cm,font=\small, rounded corners=0ex,]
\tikzstyle{argTDX}=[draw, dotted,thick, circle, inner sep=0pt,minimum size=0.6cm,font=\small]
\tikzstyle{sccX}=[draw,dotted, thick, rectangle,align=center, inner sep=0pt,minimum size=0.8cm,font=\small, rounded corners=0ex,]
\tikzstyle{argsmall}=[draw, thick, circle, fill=gray!15,inner sep=0pt,minimum size=0.4cm]
\tikzstyle{argsmallX}=[draw, thick, circle, inner sep=0pt,minimum size=0.3cm,dotted]
\tikzstyle{bag}=[draw,rectangle, rounded corners=1ex,
align=center, inner sep=1ex]

\tikzstyle{atts}=[draw,thick, inner sep=5pt, rounded corners=3pt]

\tikzstyle{nullarg}=[inner sep=0pt,outer sep=0pt,minimum size=0cm]
\tikzstyle{nullattack}=[draw, thick, |->]

\newcommand{\contrary}[1]{\overline{#1}}
\newcommand{\contraryempty}{\contrary{\phantom{a}}}

\newcommand{\theory}{\mathit{Th}}
\newcommand{\lit}{\ensuremath{\mathcal{L}}}
\newcommand{\rules}{\ensuremath{\mathcal{R}}}
\newcommand{\asm}{\ensuremath{\mathcal{A}}}

\newcommand{\cf}{\textit{cf}}
\newcommand{\adm}{\textit{adm}}
\newcommand{\com}{\textit{com}}

\newcommand{\stb}{\textit{stb}}

\newcommand{\prf}{\textit{pref}}
\newcommand{\pref}{\textit{pref}}

\newcommand{\grd}{\textit{grd}}

\usepackage{thm-restate}
\usepackage{etoolbox}

\newtheorem{example}{Example}
\newtheorem{theorem}{Theorem}
\newtheorem{definition}{Definition}
\newtheorem{observation}{Observation}

\newtheorem{principle}{Principle}

\renewcommand\thmcontinues[1]{ctd}

\title{Splitting Assumption-Based Argumentation Frameworks}

\author{%
	Giovanni Buraglio\and
	Wolfgang Dvo\v{r}\'ak\and
	Stefan Woltran\\
	\affiliations
    Institute of Logic and Computation, TU Wien, Austria\\
	\emails
	\{giovanni.buraglio, wolfgang.dvorak, stefan.woltran\}@tuwien.ac.at,
}

\begin{document}
	
	\maketitle
	
	\begin{abstract}
		Assumption-Based Argumentation (ABA) is a well-established formalism for modelling and reasoning over debates, with a wide range of applications. However, the high computational complexity of core reasoning tasks in ABA poses a significant challenge for its applicability. This issue is further aggravated when ABA frameworks (ABAFs) are instantiated into graph-based argumentation formalisms, such as Dung's Argumentation Frameworks (AFs) and Argumentation Frameworks with Collective Attacks (SETAFs). In knowledge representation and reasoning, a key strategy to address computational intractability is to optimise reasoning over a given knowledge base through divide-and-conquer algorithms. A paradigmatic example of this approach is splitting, where extensions of a given framework are computed incrementally, by restricting the search space to sub-frameworks only, and then combining the obtained results. This approach has been successfully applied to AFs, for which also a parametrised version has been introduced under stable semantics. However, the exponential growth produced by the instantiation might undermine the usefulness of splitting on the argument graphs induced by ABAFs. To address this issue, our work investigates the concept of splitting on the knowledge base rather than on its graph-based instantiation. Furthermore, we generalise splitting to its parametrised version for ABAFs.
	\end{abstract}
	
	\section{Introduction}
	Computational models of argumentation in AI~\cite{Gabbay:2021} offer formal approaches to represent and reason over debates involving conflicting or uncertain information~\cite{CarreraI15,DimopoulosMM19,FanT12,GaoTWX16,HadouxHP23,Toni13}. 
    Assumption-Based Argumentation (ABA)~\cite{BondarenkoDKT97} captures argumentative scenarios by means of so-called ABA frameworks (ABAFs), or ABA knowledge bases, consisting of a set of defeasible sentences (assumptions) and inference rules. 
	Argumentative reasoning can be performed in ABA following two different approaches. The \textit{direct} approach, typically employed by ABA solvers~\cite{LehtonenWJ21a,LehtonenWJ21b}, allows to reason over the ABA knowledge base itself through semantics defined at the level of assumptions. In contrast, the \textit{indirect} approach \cite{LehtonenRT0W24} realises reasoning in a two-step process: first an argument graph comprising arguments and their relations is generated from the ABAF, by means of the so-called \emph{instantiation} procedure; then, \emph{semantics} from abstract argumentation are applied to the obtained graph in order to find acceptable sets of arguments, along with the assumptions supporting them.
	
	Although ABA is a well-established formalism for non-monotonic reasoning, with applications in medical decision-making, explainable AI, and causal discovery~\cite{FanTMW14,Fan18,RussoRT24}, the high computational complexity of core reasoning tasks in ABA poses a significant challenge for its deployment in practice~\cite{CyrasHT21}. This issue is further aggravated when ABAFs are instantiated into argument graphs, such as Dung's Argumentation Frameworks (AFs)~\cite{Dung95} or Argumentation Frameworks with Collective Attacks (SETAFs)~\cite{NielsenP06}, as the instantiation can be computationally expensive and may result in exponentially large graphs~\cite{LehtonenR0W23}. 
	
	In the context of non-monotonic reasoning, one prominent strategy to address computational intractability is to optimise reasoning over a given knowledge base through divide-and-conquer algorithms. A paradigmatic example of this approach is \emph{splitting}, originally developed for answer-set programming~\cite{LifschitzT94} and later adapted to other nonmonotonic formalisms, e.g. default theories~\cite{Turner96} and recently Abstract Argumentation~\cite{Baumann11,BaumannBDW12,Linsbichler14,BaumannBW11,Liao13,BaroniGL14}.
	This approach focuses on incrementally computing the extensions of a given argumentation framework by means of the extension of its sub-frameworks, thereby avoiding to consider the entire solution-space of the original framework.
	
	Despite the successful application in graph-based argumentation~\cite{BaumannBW11}, splitting, so far, has been neglected for rule-based argumentation systems like Assumption-Based Argumentation. 
	A straightforward attempt to apply splitting in ABA would be to first instantiate the ABAF to a corresponding Dung's AF and then perform AF splitting. However, the exponential number of auxiliary arguments introduced during instantiation might invalidate the usefulness of splitting. 
	To overcome this limitation, we instead consider the instantiation of ABAFs as SETAFs, which yield more concise graphical representations~\cite{KoenigRU22}. 
	Thus, we first develop a splitting scheme for SETAFs, enabling incremental computation for the indirect approach of reasoning. 
	This method has the advantage of being easily applicable to any formalism that can be instantiated as a SETAF. However, its drawback is that we have to instantiate the knowledge base, which comes with a computational cost, and we cannot even use the splitting to speed up the instantiation step.
	To address this issue, we present a splitting schema that operates directly on ABA knowledge bases.
	To this end, this paper makes the following contributions:
	\begin{itemize}
		\item Towards ABA splitting, we generalise existing notions of splitting from AFs to SETAFs (Section~\ref{sec:SETAF}) for complete, stable, preferred and grounded semantics.
		\item We then introduce the notion of ABA splitting in Section~\ref{sec:ABA}, along with the syntactic adjustments required to establish a splitting theorem, which we prove under complete, stable, preferred and grounded semantics.
		\item In Section~\ref{sec:param} we extend our results to the more general notion of \emph{parameterised splitting}~\cite{BaumannBDW12}, showing that a splitting theorem holds for ABAFs under stable semantics. 
        \item In Section~\ref{sec:computeSplitting} we address computing splittings for SETAFs and ABAFs, by reducing it to a graph splitting problem and reusing dedicated algorithms for AFs. 
		\item In Section~\ref{sec:comparison} we take a closer look at the relationship between the splitting algorithm for ABAFs and the instantiation procedure into SETAFs.
		\item Finally, after reviewing related work (Section~\ref{sec:related}), we conclude with a summary and outline directions for future research (Section~\ref{sec:conc}).
	\end{itemize}
	
\section{Preliminaries}\label{sec:prel}
	\paragraph{Assumption-Based Argumentation}
    We recall here the basic concepts of assumption-based argumentation (ABA)~\cite{CyrasFST2018}. Debates are represented by means of so-called ABA Frameworks (ABAFs), which consist of a deductive system $(\lit, \rules)$, where $\lit$ is a set of sentences, and $\rules$ is a set of rules over $\lit$. A rule $r \in \rules$ has the form $a_0 \gets a_1,\dots , a_n$ with $a_i \in \lit$, $body(r) = \{a_1, \dots, a_n\}$ and $head(r) = a_0$.
	
	\begin{definition}
		An ABAF is a tuple $(\lit, \rules, \asm,\contraryempty)$, where $(\lit, \rules)$ is a deductive system, $\asm \subseteq \lit$ a set of assumptions, and $\contraryempty: \asm \to \lit$ is a total mapping, called contrary function.
	\end{definition}
	
	For a set of assumptions, $S\subseteq \asm$ we use $\contrary{S}$ to indicate the set of contraries of $S$. Conversely, we define the partial function $\alpha: \lit \to \asm$ assigning an assumption to its contrary $b\in \contrary{\asm}$ such that $\alpha(b)=a$ if $b=\contrary{a}$. This generalises to sets of contraries as before. 
	For a set of rules $R$, we fix $head(R)=\{head(r)\mid r\in R\}$, $body(R)=\{body(r)\mid r\in R\}$. Further, we use $atom(S)=\{p\in \lit \mid p \in S \lor \alpha(p)\in S \lor \contrary{p}\in S\}$. In what follows, we read $atom(p)$ as $atom(\{p\})$. 
	For a rule $r\in \rules$, we say that $r$ is: a fact if $body(r)=\emptyset$; a loop-rule if $\contrary{a}=head(r)$ and $a\in body(r)$. 
	
	A sentence $q \in \lit$ is tree-derivable from $S \subseteq \asm$ and rules $R \subseteq \rules$, denoted by $S \vdash^R q$, if there is a finite rooted labelled tree $T$ where: the root of $T$ is labelled with $q$; the set of labels for the leaves of $T$ is equal to $S$ or $S \cup \{ \top \}$; and for every inner node $v$ of $T$ there is a rule $r \in R$ such that $v$ is labelled with $head(r)$, and every successor of $v$ is labelled with $a \in body(r)$ or $\top$ if $body(r) = \emptyset$. We sometimes write $S \vdash q$ instead of $S \vdash^R q$ if it does not cause confusion. Moreover, we call $Th_D(S)=\{p\in \mathcal{L}\mid S\vdash p\}$ the theory of $S$ w.r.t.\ the ABAF $D$.
	Throughout the paper, we assume that ABAFs do not contain \emph{dummy rules}, whose body is not derivable from any set of assumptions. 
	\begin{definition}
		Let $D = (\lit, \rules, \asm,\contraryempty)$ be an ABAF. A set $S \subseteq \asm$ attacks $T\subseteq \asm$ if $S' \vdash \contrary{a}$ for some $S' \subseteq S$ and $a \in T$. A set $S$ is conflict-free in an ABAF $D$ ($S \in \cf(D)$) if it does not attack itself; $S$ defends $T$ iff it attacks each attacker of $T$ ; $S$ is closed iff $S \vdash a$ implies $a \in S$; $S$ is admissible ($S \in \adm(D)$) if it is conflict-free and defends itself.
	\end{definition}
    
	We say a set $S$ of assumptions attacks an assumption $a$ if $S$ attacks the singleton $\{a\}$.  
	Moreover, for a set $B \subseteq \asm$ we say that $S$ attacks $B$ if $S$ attacks some $b\in B$. We use $S_R^+=\{a \in \asm \mid S \vdash^R \contrary{a}\}$ and define the \emph{range} of $S$ w.r.t.\ $R\subseteq \rules$ as $S_R^\oplus=S \cup S_R^+$. 
	In this paper, we assume ABAFs to be flat, unless specified otherwise. We call an ABAF flat if every set $S$ of assumptions is closed, and \textit{non-flat} otherwise. We next recall definitions for grounded, complete, preferred, and stable ABA semantics (abbr. $\grd$, $\com$, $\prf$, $\stb$).
	\begin{definition}
		Let $D$ be an ABAF and let $S \in \adm(D)$. $S \in \com(D)$ iff $S$ contains every assumption set it defends; $S \in \grd(D)$ iff $S$ is $\subseteq$-minimal in $\com(D)$; $S \in \pref(D)$ iff $S$ is $\subseteq$-maximal in $\com(D)$; $S \in \stb(D)$ iff $S$ attacks each $\{x\} \subseteq \asm \setminus S$.
		We call $\sigma(D)$ the set of $\sigma$-extensions of $D$. 
	\end{definition}
	
	\paragraph{SETAF Instantiation}
	Reasoning in ABAFs is often performed on graphs induced from the instantiation procedure. Due to their popularity and simple structure, this step has traditionally been performed by means of Dung's AFs~\cite{Dung95,BondarenkoDKT97,CyrasFST2018}, i.e., directed graphs where nodes and their (binary) relation are interpreted as arguments and attacks among them. 
	However, this representation does not capture the possibility of having multiple assumptions attacking another directly, which might result in a large number of auxiliary arguments. Thus, in recent years, a hyper-graph representation~\cite{KoenigRU22} of ABAFs, based on argumentation frameworks with collective attacks (SETAFs)~\cite{NielsenP06}, has become increasingly popular~\cite{BuraglioD0024,RussoRT24,BertholdR024,DimopoulosDKRUW24}.
	
	\begin{definition}\label{def:setafs}
		A SETAF is a pair $SF=(A,R)$ where $A$ is 
		a finite set of \emph{arguments},
		and $R \subseteq 2^A \times A$ is the \emph{attack relation}. For an attack $(T,h)\in R$ we call $T$ the \emph{tail} and $h$ the \emph{head} of the attack.
		We write $(t,h)$ to denote the set-attack $(\{t\},h)$.
		For $S \subseteq A$, we
		say $S$ \emph{attacks} an argument $a\in A$ if there is an attack $(T,a)\in R$ with $T\subseteq S$. Moreover, for a set $B \subseteq A$ we say that $S$ attacks $B$ if $S$ attacks some $b\in B$. We use $S_R^+=\{a \in A \mid S \text{ attacks } a\}$ and define the \emph{range} of $S$ w.r.t.\ $R$ as $S_R^\oplus=S \cup S_R^+$.  
		By AFs we refer the class of SETAFs where all attacks $(T,h)\in R$ are such that $|T|=1$.
	\end{definition}
	
	For all semantics under our considerations, an ABAF $D =  (\lit, \rules, \asm,\contraryempty)$ can be instantiated as the (equivalent) SETAF $SF_D=(A_D,R_D)$ where $A_D = \asm$ and $(S,a) \in R_D$ iff $S \vdash \contrary{a}$ \cite{KoenigRU22}.
	
	\begin{example}\label{ex:prel}
For the ABAF $D =(\lit, \rules, \asm,\contraryempty)$ (left) and its corresponding SETAF $SF_D$ (right), it holds that $\{a,w,z\}\in \prf(D)=\prf(SF_D)$.
    
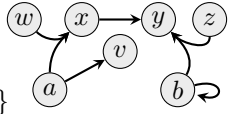
    \begin{minipage}{.28\textwidth}
	\hspace{-18pt}
	\begin{tabular}{cl}
		$\asm=$ &$ \{a,b,v,w,x,y,z\}$\\
		$\lit=$ &$\asm \cup \contrary{\asm} \cup \{p\}$\\
		$\rules=$&$\{\contrary{b} \gets b, \; p\gets a, \; \contrary{v}\gets a,$ \\
		& \; $\contrary{x}\gets p,w, \;  \contrary{y}\gets x, \; \contrary{y}\gets \contrary{b},z\}$ 
	\end{tabular}
\end{minipage}
\begin{minipage}{.165\textwidth}
\begin{tikzpicture}[scale=0.85,>=stealth]
	\path
	(1.5,-0.1) node[arg] (a){$a$}
	(3.5,-0.1) node[arg](b){$b$}
	(2.6,0.5) node[arg](v){$v$}
	(1.1,1) node[arg] (w) {$w$}
	(2,1) node[arg] (x) {$x$}
	(3.2,1) node[arg] (y) {$y$}
	(4,1) node[arg] (z) {$z$}
	;
	
	\path[->,>=stealth,thick]
	(a) edge (v)
	(x) edge (y)
	(b) edge[loop right] (b)
	(b) edge[out=60,in=-45] (y)
	(z) edge[out=-90,in=-45] (y)
	(w) edge[out=-45,in=-135] (x)
	(a) edge[out=90,in=-135] (x)
	;		
\end{tikzpicture}
\end{minipage}
    \end{example}
	
	Notice that such a mapping is many-to-one, i.e., several ABAFs correspond to the same SETAF. This is because the SETAF instantiation only remembers the attacks between assumptions but forgets all non-assumptions.
	In the above example, we lose the sentence $p$ when instantiating the derivation $\{a,w\}\vdash \contrary{x}$ build via $\contrary{x}\gets p,w$ and $p\gets a$ into $(\{a,w\},x)$.
	On the other hand, SETAFs can be seen --- syntactically --- as a fragment of flat ABAFs, by modelling each attack $(S,a)$ via a dedicated rule $\contrary{a}\gets S$.

	\paragraph{Splitting of AFs}
	We now recall the splitting approach for AFs~\cite{Baumann11}. A splitting identifies two sub-frameworks $F_1$ and $F_2$ separated by a set of attacks going from $F_1$ to $F_2$. Then, the information contained in an extension of $F_1$ is propagated, computing the so-called \textit{reduct} of $F_2$ accordingly.
	\begin{definition}\label{def_af_split1}
		Let $F=(A,R)$ be an AF, 
		$F_1=(A_1,R_1)$ and $F_2=(A_2,R_2)$ two sub-frameworks of $F$ 
		s.t.\  $A_1\cap A_2=\emptyset$, 
		$A=A_1 \cup A_2$ and 
		$R=R_1 \cup R_2 \cup R_3$ 
		with $R_3\subseteq A_1 \times A_2$.
		The triple $(F_1,F_2,R_3)$ is called a \emph{splitting} of $F$.
		For such a splitting and a set $E\subseteq A_1$, the \emph{$(E,R_3)$-reduct} is the AF $AF'=(A',R')$ with $A'=A_2\setminus E^+_{R_3}$ and $R'=R_2\cap (A'\times A')$. 
		Moreover, the set of \emph{undecided arguments} w.r.t.\ $E\subseteq A_1$ is $U_E=A_1\setminus E^\oplus_{R_1}$.
	\end{definition}
	The reduct is designed to take care of the arguments attacked by the set $E$.
	To account for the propagation of undecided arguments w.r.t.\ $E$, a further \emph{modification} is needed: self-attacks are propagated from $F_1$ to arguments in~$F_2$.
	\begin{definition}\label{def_af_split2}
		Let $(F_1,F_2,R_3)$ be a splitting for an AF $F$ and $E$ an extension of $F_1$. Moreover, take $F'_2= (A'_2, R'_2)$ as the $(E,R_3)$-reduct of $F_2$ and $U_E$ as the set of undecided arguments w.r.t.\ $E$. The $(U_E,R_3)$-modification of $F_2$ is: 
		$$ mod_{U_E}^{R_3}(F'_2)=(A'_2, R'_2\cup \{(b,b)\mid \exists a\in U_E : (a,b)\in R_3\}).$$
	\end{definition}
	
	Using these definitions,~\citeauthor{Baumann11}~(\citeyear{Baumann11}) has shown that it is possible to split the AF and compute the extensions for each sub-framework incrementally such that their combination yields extensions of the original framework.
	\begin{theorem}{\cite{Baumann11}}
		Let $(F_1,F_2,R_3)$ be a splitting for an AF $F=(A,R)$ and $\sigma \in \{\adm,\stb,\com,\prf,\grd\}$. 
		\begin{enumerate}
			\item If $E_1\in \sigma(F_1)$ and $E_2\in \sigma(mod_{U_E}^{R_3}(F'_2))$, then $E_1\cup E_2\in \sigma(F)$.
			\item If $E\in \sigma(F)$, then $E\cap A_1\in \sigma(F_1)$ and $E\cap A_2\in \sigma(mod_{U_E}^{R_3}(F'_2))$.
		\end{enumerate}
	\end{theorem}
	
	Later, this idea has been generalised by relaxing the strict separation requirement, which significantly narrows the applicability of splitting, introducing so-called \textit{parametrised splitting}~\cite{BaumannBDW12}. Instead of demanding that the first part is completely unaffected by the second, it allows some forms of interaction. This generalisation is captured by the notion of \emph{quasi-splitting}, where arguments in $F_1$ may be externally attacked by arguments in $F_2$. The goal is to preserve correctness while broadening the applicability of splitting. This is achieved by enriching $F_1$ with meta-information that encodes facts about potential influences (e.g.\ attacks) from the second sub-framework. 
	In particular, for each externally attacked argument $a$, a fresh argument $a'$ is added to $F_1$ along with a symmetric attack on $a$, enforcing a choice between $a$ and $a'$ in $F_1$. Then, $F_2$ is modified accordingly: the previous choices are propagated in the second sub-framework via the reduct as well as additional nodes and attacks. Stable extensions of the entire AF are then recovered by composing compatible solutions from the two modified sub-frameworks.

	\section{Splitting Collective Attacks}\label{sec:SETAF}
	
	Towards a splitting approach for ABA we first introduce a splitting scheme for argumentation frameworks with collective attacks (SETAFs).
	First, due to their correspondence with ABAFs this paves the way for the splitting scheme on the ABA knowledge base that we will introduce in the next second.
	Second, this provides a divide-and-conquer methodology to enhance existing solvers for SETAFs~\cite{DvorakGW18,GresslerDW24}, and thus allows for splitting when solving ABAFs via instantiations to SETAFs.
	
	We introduce a notion of splitting for SETAFs that generalises the one for Dung-style AFs. 
	
	\begin{definition}
		\label{def:setaf splitting}
		Let $SF=(A,R)$ be a SETAF, $SF_1=(A_1,R_1)$ and $SF_2=(A_2,R_2)$ two sub-frameworks of $SF$ such that $A_1\cap A_2=\emptyset$, $A=A_1 \cup A_2$ and $R=R_1 \cup R_2 \cup R_3$ with $R_3\subseteq \big((2^{A_1}\setminus \{\emptyset\}  )\cup 2^{A_2}\big) \times A_2$. We call a \emph{splitting} of $SF$ the triple $(SF_1,SF_2,R_3)$. Moreover, we call $R_3$ the set of \emph{links} w.r.t.\ $(SF_1,SF_2,R_3)$ and say that a link is \emph{undecided} if no argument in its tail is defeated (i.e.\ attacked by an extension), but at least one is undecided.
	\end{definition}
	
	As for AFs, the general idea is to compute extensions of $SF$ as a combination of extensions of $SF_1$ and $SF_2$.
	Due to the links from $SF_1$ to $SF_2$ we have to modify $SF_2$ according to the extension(s) of $SF_1$ to account for the prior accepted and rejected arguments.
	Following~\citeauthor{Baumann11}~(\citeyear{Baumann11}), we introduce the notions of \emph{reduct} and \emph{modification}, in application to the second part (that is, $SF_2$) of the original SETAF.
	Intuitively, the reduct takes care of the arguments in $SF_2$ that are already defeated by $E_1$ by deleting them. Moreover, the reduct removes an attack in $R_2$ if one of
	the arguments in the tail or the head is defeated by $E_1$.
	The attacks of $R_3$ are first filtered and the simplified by the reduct in order to fit to the new argument set. An attack is neglected if either a tail or the head argument is defeated by $E_1$, all the tail arguments are in $A_1$, or there is a
	tail arguments in $A_1$ that is not accepted by $E_1$. The remaining attacks are simplified by removing the arguments of $A_1$ from the tail of the attack.

	\begin{definition}[Reduct]\label{def_reduct}	
		Let $(SF_1,SF_2,R_3)$ be a splitting for a SETAF $SF$. We define the \emph{$(E_1,R_3)$-reduct} (or simply reduct) of $SF_2$ for some extension $E_1$ of $SF_1$ as the SETAF $SF'_2=(A'_2,R'_2)$ with: 
		\begin{align*}
			A'_2=&\{a\in A_2\mid a \notin (E_1)_{R_3}^+\};\\ 
			R'_2=&
			\{ (T,h)\in R_2\mid T\subseteq A'_2,h\in A'_2  \}
			\;\cup \\
			&\{(T\setminus A_1, h) \mid (T,h)\in R_3,\ T \setminus A_1 \neq \emptyset, \\
			&\hspace{38pt} T\cap A_1\subseteq E_1, T\cap (E_1)^+_{R_3}=\emptyset,\ h\in A_2'\}.
		\end{align*}
	\end{definition}

	When dealing with undecidedness, what guides our intuition towards a certain modification is not the status of the arguments in $SF_1$, but rather the status of the \emph{links}.
	Hence, we slightly tweak the original definition and base our notion solely on the undecided \emph{links}.
	\begin{definition}[Undecided Links]
		Given a splitting $(SF_1,SF_2,R_3)$ for a SETAF $SF$ and an extension $E_1\in SF_1$ we define the \emph{set of undecided links} w.r.t.\ $E_1$ as:
		\begin{align*}
			U^{E_1}_{R_3}= \; &   \{(T,h)\in R_3 \mid T\cap (E_1)^+_{R_1\cup R_3}=\emptyset,  \\
			& \hspace{62pt} \exists t\in T: t\in A_1\setminus (E_1)_{R_1}^\oplus \}.
		\end{align*}	
	\end{definition}
	
	In what follows, we define the \emph{modification}, which is applied on the reduct, and accounts for the effects of the undecided links. 
	In particular, for each undecided link $(S,t) \in R_3$ we add to $F'_2$ a (set-)self-attack from $t$ (together with the $F'_2$-part of the attack) to itself.     
	\begin{definition}[Modification]\label{def_mod}
		Let $(SF_1,SF_2,R_3)$ be a splitting for a SETAF $SF$ and $E_1$ an extension of $SF_1$. Take $SF'_2$ as the $(E_1,R_3)$-reduct of $SF_2$ and $U^{E_1}_{R_3}$ as the set of undecided links w.r.t.\ $E_1$. We denote with $mod^{E_1}_{R_3}(SF'_2)=SF^{\star}_2=(A^{\star}_2,R^{\star}_2)$ the $U^{E_1}_{R_3}$-modification (or simply modification) of $SF'_2$ s.t. $A^{\star}_2=A'_2$ and $R^{\star}_2$ is given by:
		\begin{align*}
			R'_2\cup \{((T \cap A'_2) \cup \{h\}, h) \mid (T,h)\in U^{E_1}_{R_3}, \; h\in A_2'\}.
		\end{align*}
	\end{definition}
	
	Before we present the splitting theorem we illustrate Definitions~\ref{def_reduct}--\ref{def_mod} in the following example.
	\begin{example}
		\label{ex_big_example}
        Consider the SETAF $SF$ of Example~\ref{ex:prel} that separates arguments in $A_1=\{a,b\}$ from those in $A_2=\{v,w,x,y,z\}$. We select the preferred extension $E_1=\{a\}$ of $SF_1$, and compute the reduct $SF'_2$. This is obtained by: (i) deleting $v$, as it is attacked by $E_1$; (ii) projecting $(\{a,w\},x)$ to $(\{w\},x)$ since $a\in E_1$; and (iii) neglecting the attack $(\{b,z\},y)$ because $b\notin E_1$. Subsequently, we construct the modification $SF^{\star}_2$. For this, case (iii) is crucial: the attack $(\{b,z\},y)$ is an \emph{undecided link} since $b$ is not in the range of $E_1$ and $\{b,z\}$ is not attacked by it. Thus, we introduce set-self-attack $(\{y,z\},y)$ to $SF'_2$. 
	
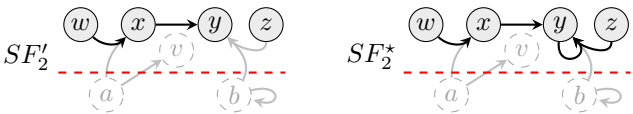
\begin{center}
		\begin{tikzpicture}[scale=0.85,>=stealth]
			\begin{scope}[shift={(0,0)}]
				\path
				(0.25,0.5) node () {$SF'_2$}
			(1.5,-0.1) node[argd] (a){$a$}
			(3.5,-0.1) node[argd](b){$b$}
			(2.6,0.6) node[argd](v){$v$}
			(1.1,1) node[arg] (w) {$w$}
			(2,1) node[arg] (x) {$x$}
			(3.2,1) node[arg] (y) {$y$}
			(4,1) node[arg] (z) {$z$}
			;
			
			\path[->,>=stealth,thick]
			(a) edge[gray!50] (v)
			(x) edge (y)
			(b) edge[loop right,gray!50] (b)
			(b) edge[out=60,in=-45,gray!50] (y)
			(z) edge[out=-90,in=-45,gray!50] (y)
			(a) edge[out=90,in=-135,gray!50] (x)
			(w) edge[out=-45,in=-135] (x)
			;		
			\draw [thick,dashed,red] (0.75,0.25) -- (4.3,0.25);
		\end{scope}
		\begin{scope}[shift={(5.4,0)}]
	\path
		(0.25,0.5) node () {$SF^{\star}_2$}
(1.5,-0.1) node[argd] (a){$a$}
(3.5,-0.1) node[argd](b){$b$}
(2.6,0.6) node[argd](v){$v$}
(1.1,1) node[arg] (w) {$w$}
(2,1) node[arg] (x) {$x$}
(3.2,1) node[arg] (y) {$y$}
(4,1) node[arg] (z) {$z$}
;

\path[->,>=stealth,thick]
(a) edge[gray!50] (v)
(x) edge (y)
(b) edge[loop right,gray!50] (b)
(b) edge[out=60,in=-45,gray!50] (y)
(z) edge[out=-90,in=-45] (y)
(y) edge[out=-95,in=-45,loop,looseness=5] (y)
(a) edge[out=90,in=-135,gray!50] (x)
(w) edge[out=-45,in=-135] (x)
;		
\draw [thick,dashed,red] (0.75,0.25) -- (4.3,0.25);
		\end{scope}	
		\end{tikzpicture}
	\end{center}
Finally, we obtain $E_2=\{w,z\}$ as the preferred extension of $SF^{\star}_2$, retrieving $E=E_1\cup E_2=\{a,w,z\}$ as a preferred extension of $SF$.	 

	\end{example}
	Having these notions at hand, we now establish the adequacy of the splitting technique for SETAFs.
	We start by establishing that (a) conflict-freeness of the sub-frameworks $SF_1$ and $SF_2$ carries over to the whole SETAF $SF$,
	and (b) conflict-free sets of $SF$ induce conflict-free subsets in $SF_1$ and $SF_2'$.
	\begin{restatable}{proposition}{SETAFsplittingCF}\label{prop: cf splitting}
		Let $(SF_1,SF_2,R_3)$ be a splitting for a SETAF $SF=(A,R)$ with $SF_1=(A_1,R_1)$ and $SF_2=(A_2,R_2)$. Let $SF'_2=SF^{E_1}_2$.
		\begin{enumerate}
			\item If $E_1\in \cf(SF_1)$ and $E_2\in \cf(mod^{E_1}_{R_3}(SF'_2))$, then $E_1\cup E_2\in \cf(SF)$. 
			\item If $E\in \cf(SF)$, then $E_1=E\cap A_1\in \cf(SF_1)$ and $E\cap A_2\in \cf(SF'_2)$.
		\end{enumerate}
	\end{restatable}
	
	Finally, we are ready to characterize the splitting algorithm by generalising the splitting theorem for SETAFs under the standard Dung semantics. 
	\begin{restatable}{theorem}{SETAFsplitting}\label{theorem: SETAF splitting}
		Let $(SF_1,SF_2,R_3)$ be a splitting for a SETAF $SF=(A,R)$ with $SF_1=(A_1,R_1)$, $SF_2=(A_2,R_2)$,  and $\sigma \in \{\stb,\adm,\com,\prf,\grd\}$. Further, let $SF_2^{\star}=mod^{E_1}_{R_3}(SF'_2)$. 
		\begin{enumerate}
			\item If $E_1\in \sigma(SF_1)$ and $E_2\in \sigma(SF_2^{\star})$, then $E_1\cup E_2\in \sigma(SF)$. 
			\item If $E\in \sigma(SF)$, then $E_1=E\cap A_1\in \sigma(SF_1)$ and $E\cap A_2\in \sigma(SF_2^{\star})$.
		\end{enumerate}
	\end{restatable}
	
	While the existing instantiation procedure from ABA frameworks to SETAFs provides a foundation for defining splitting, attempting to directly replicate the SETAF-style idea of splitting among assumptions fails to yield a natural notion of splitting.
	This disconnect stems from a fundamental structural difference: in SETAFs, attacks are primitive, whereas in ABA, they are derived from the underlying deductive system 
	$(\lit,\rules)$. As a result, naively mimicking SETAF-style splitting in ABA would require (i) arbitrarily partitioning the assumption set into $\asm_1$ and $\asm_2$, and (ii) computing attacks as derivations from assumptions in $\asm_1$ to those in $\asm_2$.
	However, splitting should be possible solely by inspecting the knowledge base at hand. 
	Moreover, while instantiating ABAFs into SETAFs has been shown useful in specific contexts \cite{RussoRT24,BuraglioD0024}, this approach comes with a critical drawback: it can yield an exponential growth in the number of collective attacks generated, thus increasing in input size. This inefficiency motivates many ABA solvers to operate directly on ABAFs rather than relying on their abstract representations.
	Therefore, to enable an efficient form of splitting, we propose a dedicated splitting algorithm tailored to the syntactic structure of ABAFs.

	\section{Splitting in Assumption-Based Argumentation}\label{sec:ABA}
	In this section we present splitting results for ABAFs.  
	The rule-set of an ABAF is split into a bottom and a top part whenever no assumption occurs in the bottom part whose contrary is derived by some rule in the top. 
	This ensures that the assumptions in the bottom can be evaluated independently of what can be deduced by inspecting the top part. We capture this intuition via the notion of splitting set:
	\begin{definition}\label{def: spl aba}
		Given an ABAF $D=(\lit,\rules,\asm,\contraryempty)$, a set $S\subseteq \lit$ is a splitting set (or simply a splitting) of $D$ if $S=atom(S)$ and for all $r \in \mathcal{R}$, $head(r)\in S$ \text{ implies } $body(r)\subseteq S$. 
	\end{definition}
	
	A splitting set partitions the deductive system into two sub-systems $(\lit_1$,\rules$_1)$ and $(\lit_2$,\rules$_2)$, called the `bottom' and `top'. In particular, we have (i) $\lit_1=S$ and $\rules_1=\{r \in \rules \mid head(r)\in S\}$ and (ii) $\lit_2=\lit\setminus S$ and $\rules_2=\{r \in \rules \mid head(r)\notin S\}$. These induce respectively $D_1=(\lit_1,\rules_1,\asm_1,\contraryempty^1)$ and $D_2=(\lit_2 ,\rules_2,\asm_2,\contraryempty^2)$ with $\asm_i=\lit_i\cap \asm$ and $\contraryempty^i$ defined over $\asm_i$. 
	
	\begin{example}\label{ex: big ABAF}
		Consider again the ABAF $D=(\lit,\rules,\asm,\contraryempty)$ from Example~\ref{ex:prel}, with assumptions $\asm=\{a,b,v,w,x,y,z\}$, sentences $\lit=\asm \cup \contrary{\asm} \cup \{p\}$, and rules $\rules$ as follows: 
		\begin{align*}
            \contrary{b}& \gets b && p\gets a &&
            \contrary{v} \gets a && \contrary{x} \gets w,p && \contrary{y} \gets x && \contrary{y}\gets \contrary{b},z
		\end{align*}
		Take the set $S=\{a,b,\contrary{a},\contrary{b},p\}$.
        It can be easily checked that $S$ is a splitting set of $D$, through which we obtain two sub-systems $(\lit_1$, \rules$_1)$ (bottom) and $(\lit_2$,\rules$_2)$ (top) such that:
        \begin{itemize}
            \item $\lit_1=S$ and $\lit_2=\lit\setminus S$;
            \item $\rules_1=\{\contrary{b} \gets b, \; p\gets a\}$ and $\rules_2=\rules\setminus \rules_1$. 
        \end{itemize}
	\end{example}
	
	Notice that some atoms contained in $\lit_1$ (but not in $\lit_2$) may occur in the body of some rule in $\rules_2$ ($a$, $\contrary{b}$ and $p$ in Example \ref{ex: big ABAF}). This intermediate mismatch will be resolved later by the notion of \textit{reduct}. Moreover, their occurrences in $\rules_2$ does not affect the acceptance status of such atoms.   
    In fact, a first sanity check, we observe that our notion of splitting prevents building attacks from assumptions of $D_2$ towards assumptions of $D_1$ using top-rules in $\rules_2$. This is ensured by the fact that contraries of assumptions occurring in the bottom part are not derived via rules in the top part (via construction of $\rules_2$). 
    As a result, assumptions in $\asm_1$ are attacked only via rules in $\rules_1$ by assumptions in $\asm_1$. Thus, no attack generated from $\asm_2$ (by means of rules in $\rules_2$) is directed towards $\asm_1$. 
	
	\begin{restatable}{proposition}{ABAattacksConserv}\label{pro:att1-conserv}
		Let $D$ be an ABAF and $S$ a set of literals that splits $D$ into $D_1$ and $D_2$. For every derivation $T\vdash^R \contrary{a}$ with $a\in \asm_1$, it holds that $R\subseteq \rules_1$ and $T\subseteq \asm_1$. 
	\end{restatable}
	
	The attacks of the bottom part can be extended in a conservative way: whatever happens in the second sub-framework does not affect the acceptability of assumptions in $D_1$.   
	Thus, to compute incrementally an extension of an ABAF $D$, we can first select an extension $E$ of $D_1$ and later modify $D_2$ according to the information contained in $E$. Consequently, we can evaluate the modified framework $D_2$ and augment its extensions with $E$. 
	Again, we follow the approach of Baumann and appeal to the notions of \textit{reduct} and \textit{modification} in application to $D_2$ via a two-step process. 
	First, we propagate all the information we get from a $\sigma$-extension $E$ of $D_1$ to ensure that rules in contrast with $E$ are removed. The outcome is called the $E$-reduct of $D_2$.
	
	\begin{definition}
		Let $D=(\lit,\rules,\asm,\contraryempty)$ be an ABAF, $S$ a splitting set of $D$ into $D_1$ and $D_2$ and $E$ a $\sigma$-extension of $D_1$. We call $D^E_2=(\lit_2,\rules^E_2,\asm_2,\contraryempty^2)$ the $E$-reduct (or simply reduct) of $D_2$, where $\rules^E_2$ is obtained by deleting:
		\begin{itemize}
			\item each rule $r\in \rules_2$ with $body(r)\cap S \not \subseteq \theory_{D_1}(E)$; 
			\item all literals in $\theory_{D_1}(E)$ from the remaining rules. 
		\end{itemize}
	\end{definition}
	
	As we anticipated, the rule-set $\rules^E_2$ now contains all and only those atoms occurring in $\lit_2$. Therefore, the reduct can be evaluated in complete isolation from $D_1$. 
	In the second step, we modify the reduct to propagate the information about assumptions (or their contraries) which are not contained in $\theory_{D_1}(E)$. We call these assumptions \textit{undecided}, as they are not in $E$ nor their contrary is derivable from it (i.e. are not attacked by $E$). Then, the  \textit{set of undecided assumptions} of $D_1$ w.r.t.\ $E$ is $\mathsf{UA}_{D_1}(E)=\{a \in \asm_1 \mid a \notin E \text{ and } \contrary{a} \notin \theory_{D_1}(E)\}$. Since their status can be transmitted to other assumptions via rules, we define the concept of \emph{undecided theory} of $D_1$, capturing all sentences derived from some undecided (but no defeated) assumptions.  
	
	\begin{definition}
		Let $D=(\lit,\rules,\asm,\contraryempty)$ be an ABAF and $E\in \sigma(D)$. 
		The \emph{undecided theory} $\mathsf{UT}_{D}(E)$ of $D$ w.r.t. $E$ is the set of sentences $p\in \lit$ for which there is a $ T\subseteq \asm$ such that
		\begin{inparaenum}[(i)]
			\item $ T\vdash p$,
			\item $T\cap \mathsf{UA}_{D}(E)\neq \emptyset$, and
			\item $\contrary{T}\cap Th_{D}(E)=\emptyset$.
		\end{inparaenum}
	\end{definition}
	
	Rules in $D_2$ whose bodies contain elements of $\mathsf{UT}_{D_1}(E)$ might carry over undecidedness from $D_1$. However, this scenario could be overwritten by the presence of \textit{incompatible sentences} w.r.t.\ $E$, captured by $\mathsf{IS}_{D_1}(E)=Th_{D_1}(E^+_{\rules_1}) \cup \contrary{E}$, where $E^+_{\rules_1}=\{a\in \asm_1 \mid E\vdash^R \contrary{a}, R\subseteq \rules_1\}$. 
	Hence, a set of sentences from $D_1$ will carry undecidedness to sentences in $D_2$ if and only if (i) none of its elements is incompatible and (ii) at least one of its elements is in the undecided theory w.r.t. the previously selected extension. 
	This concept mirrors the notion of undecided links for SETAFs. 
	
	We are now in the position to formally define the \textit{modification} of $D^E_2$.  
	First, we expand the set of sentences with a fresh assumption $x_u$ and corresponding contrary. Further, we introduce (i) a loop-rule for $x_u$ and (ii) a modified version of every rule with some undecided (but no incompatible) sentence in the body. In particular, we expand their body with $x_u$, after projecting to $\lit_2$. 
	
	\begin{definition}\label{def: aba-modif}
		Let $D$ be an ABAF, $S$ a set that splits $D$ into two sub-frameworks $D_1$ and $D_2$ and $E$ an extension of $D_1$. Further, let $D'_2$ be the $E$-reduct of $D_2$. We use $mod^E_{D_1}(D'_2)= D^{\star}_2=(\lit^\star_2,\rules^\star_2,\asm^\star_2,\contraryempty^\star)$ to denote the $E$-modification (or simply modification) of $D'_2$ where $D^{\star}_2=D'_2$ if $\mathsf{UA}_{D_1}(E)=\emptyset$, and otherwise: 
		\begin{align*}
			&\lit^\star_2=  \lit_2 \cup \{x_u, \contrary{x_u}\}; \\ 
			&\rules^\star_2=  \rules'_2  \cup \{\contrary{x_u}\gets x_u\} \ \cup \\
            &\hspace{23pt} \{head(r)\gets (body(r)\cap \lit_2) \cup \{x_u\}  \mid r\in \rules_2,\\
			& \hspace{29pt} body(r)\cap \mathsf{IS}_{D_1}(E)=\emptyset, body(r) \cap \mathsf{UT}_{D_1}(E)\neq \emptyset \}.
		\end{align*}
	\end{definition}
	
	\begin{example}
		Consider again the ABAF $D=(\lit,\rules,\asm,\contraryempty)$ from Example \ref{ex: big ABAF} and splitting set $S=\{a,b,\contrary{a},\contrary{b},p\}$. We know that $\{a\}\in \prf(D_1)$. Therefore, the $\{a\}$-reduct of $D_2$ is $D_2^{\{a\}}=(\lit_2,\rules^{\{a\}}_2,\asm_2,\contraryempty^2)$, where the rule-set $\rules^{\{a\}}_2$ is:
		\begin{align*}
            \contrary{v} & \gets \textcolor{gray!70}{a} && \contrary{x} \gets w\textcolor{gray!70}{,p} && \contrary{y} \gets x && \textcolor{gray!70}{\contrary{y}\gets \contrary{b},z} 
		\end{align*}
		
		Moreover, $\mathsf{UA}_{D_1}({\{a\}})=\{b\}$ is the set of undecided assumptions and $UT_{D_1}(\{a\})=\{b,\contrary{b}\}$. We then compute the modification $D_2^{\star}$ w.r.t. $\{a\}$ by expanding the set of sentences $\lit_2$ with $\{x_u, \contrary{x_u}\}$ and the rule-set $\rules^{\{a\}}_2$ above such that: 
		\begin{align*}
            \contrary{v} & \gets && \contrary{x} \gets w && \contrary{y} \gets x && \contrary{y}\gets x_u,z && \contrary{x_u}\gets x_u
		\end{align*}
        Since $\{w,z\}\in \prf(D^{\star}_2)$, we get $\{a,w,z\}$ as in Example \ref{ex_big_example}.  
	\end{example}
	We can now prove that our procedure preserves conflict-free sets under incremental computation as well as projection to sub-frameworks, as for SETAFs in Section \ref{sec:SETAF}.
	\begin{restatable}{proposition}{ABAsplittingCF}\label{pro:cf-ABA}
		Let $S\subseteq \mathcal{L}$ be a splitting set of an ABAF $D$ into $D_1$ and $D_2$. Moreover, let $D^{\star}_2=mod_{D_1}^{E_1}(D^{E_1}_2)$. 
		\begin{enumerate}
			\item If $E_1 \in \cf(D_1)$ and $E_2\in \cf(D^{\star}_2)$, then $E_1 \cup E_2 \in \cf(D)$.
			\item If $E\in \cf(D)$, then $E\cap \asm_1 \in \cf(D_1)$ and $E\cap \asm_2 \in \cf(D^E_2)$. 
		\end{enumerate}
	\end{restatable}
	
	We then show that splitting is adequate with respect to stable, admissible, complete, preferred and grounded semantics. Due to space constraints we include proof details only for admissible semantics, being
	prototypical for the others. 
	
	\begin{restatable}{theorem}{ABAsplitting}\label{theorem: ABA splitting}
		Let $S$ be a splitting set for an ABAF $D$ into $D_1$ and $D_2$ and $\sigma=\{\stb, \adm,\com, \prf, \grd\}$. Further, let $mod_{D_1}^{E_1}(D^{E_1}_2)=D^{\star}_2$.
		\begin{enumerate}
			\item If $E_1 \in \sigma(D_1)$ and $E_2\in \sigma(D^{\star}_2)$, then $E_1 \cup E_2 \in \sigma(D)$.
			\item If $E\in \sigma(D)$, then $E\cap \asm_1 \in \sigma(D_1)$ and $E\cap \asm_2 \in \sigma(D^{\star}_2)$. 
		\end{enumerate}
	\end{restatable}
	\begin{proof}
		\textbf{(admissible)}. (1.) Since admissibility implies conflict-freeness from Proposition \ref{pro:cf-ABA}, we know that $E=E_1\cup E_2\in \cf(D)$. Thus we only need to show that $E$ defends itself in $D$, i.e. for all $a\in E$, if $T\vdash\contrary{a}$, then $T'\vdash \contrary{t}$ for some $t\in T$ and $T'\subseteq E$.
		If $a\in E_1$, we know that $a$ is defended by $E_1$ in $\asm_1$ from hypothesis. Thus, from Proposition \ref{pro:att1-conserv}, we can deduce that $E_1\in \adm(D)$. 
		Consider now an assumption $a \in E_2$ and some $T\subseteq \asm$ such that $T\vdash^{R}\contrary{a}$ and $R\subseteq \rules$. If $\contrary{T}\cap Th_{D_1}(E_1)\neq \emptyset$, then $E_1$ defends $a$ against $T$ in $D$. If $\contrary{T}\cap Th_{D_1}(E_1)=\emptyset$, this means that $T\subseteq \asm_2$ and $T \vdash \contrary{a}$ in $D'_2$ ($a$ is attacked in the reduct) or $T\cup x_u\vdash^R \contrary{a}$ in $D_2^{\star}$ ($a$ is attacked in the modification). 
		In both cases, since $E_2$ is conflict-free and defends $a$ in $D_2^{\star}$, there is a $T'\subseteq E_2$ such that $T'\vdash \contrary{t}$ with $t\in T$. We distinguish two cases: either (i) $T'\vdash \contrary{t}$ already in $D_2$, in which case $a$ is defended by $E$ in $D$, or (ii) there is some $T''\supset T'$ such that $T''\vdash\contrary{t}$ in $D$ and $T''\cap \asm_1\subseteq E_1$. Thus, since $T\subseteq E_1\cup E_2$, $a$ is defended by $E_1\cup E_2$ in $D$.  In any case $a$ is defended in $D$ by $E$. 

		(2.) By Proposition \ref{pro:cf-ABA}, we get $E_1=E\cap \asm_1\in \cf(D_1)$ and $E_2=E\cap \asm_2\in \cf(D'_2)$. First, we know that $E_1=E\cap \asm_1\in \adm(D_1)$ because $E$ defends itself in $D$ and $E\cap \asm_1$ is not attacked by a subset of $\asm_2$ (Proposition \ref{pro:att1-conserv}). It remains to prove that $E_2=E\cap \asm_2\in \adm(D^{\star}_2)$. Take an assumption $a\in E_2$ such that $T\vdash\contrary{a}$ in $D^{\star}_2$. Each such derivation corresponds to exactly one derivation $T'\vdash^{R} \contrary{a}$ with $R\subseteq \rules$. There are two cases: either (i) $T'=T\subseteq \asm_2$ and $R\subseteq \rules_2$ or (ii) $T'\supset T\setminus \{x_u\}$ where $T'\setminus T\subseteq E_1$ (assumptions deleted from simplified rules in the reduct). From both (i) and (ii) we deduce that $\contrary{T'}\cap Th_{D_1}(E_1)=\emptyset$: for (i) because it would entail $T\not\subseteq \asm_2$; for (ii) because otherwise $T\not \vdash\contrary{a}$ in $D^{\star}_2$. Nonetheless, since $E$ defends $a$ in $D$, in case (i) there is a counter-attack $T''\vdash^{\rules_2} \contrary{t}$ such that $T''\subseteq E$ and $t\in (T\setminus \{a\})$. In case (ii), the same holds but $t\in (T'\setminus \{a,x_u\})$. If $T''\cap \asm_1=\emptyset$, we know that $\{t\} \subseteq \asm'_2$ and together with the fact that $T''\subseteq E$, we derive $T''\subseteq E\cap \asm'_2=E_2$. Hence, $T''$ defends $a$ from $T$ in $D^{\star}_2$. If $T''\cap \asm_1 \neq \emptyset$, then $T''\cap \asm_1\subseteq E_1$. Therefore, from $T''\subseteq E$ we get $T''\cap \asm'_2 \vdash^{\rules'_2} \contrary{t}$, which defends $a$ against $T$ in $D^{\star}_2$. Thus $a$ is always defended in $D^{\star}_2$, as desired. 
	\end{proof}
	
	The above results guarantees the correctness of our divide-and-conquer algorithm. However, as anticipated in the introduction, splitting is only possible when certain structural properties are met. For graph-based frameworks, it is required that all attacks between the two sub-frameworks share the same direction. For this reason, the presence of a single attack can invalidate the possibility of splitting (e.g.\ when the AF in question has the form of a Strongly Connected Component (SCC)). In the case of ABA, the same can happen when an assumption $b$ in the splitting set $S$ is attacked by assumptions outside of it. In such cases, splitting is impossible, as the bottom part of ABAF is not independent from the top. Thus, the acceptability of $b$ cannot be ensured by projecting to the bottom. Similarly to graph-based frameworks, the possibility to make use of a splitting algorithm may be hindered, even in cases where only one out of many assumptions in $S$ is externally attacked. To overcome such issue, we consider in the following section a generalised version of ABA splitting. 
	
	\section{Parametrised Splitting}\label{sec:param}
	We now introduce a more general version of splitting for ABAFs, called \emph{parametrised splitting} inspired by~\citeauthor{BaumannBDW12}~(\citeyear{BaumannBDW12}). This relaxes the structural constraint for the application of splitting, potentially allowing a fixed number of assumptions in the bottom part to be attacked from assumptions in the top. 
	Precisely, in contrast with the previous notion of splitting, we allow some contraries of assumptions occurring in bodies of $\mathcal{R}_1$ to appear as the heads of rules in $\mathcal{R}_2$. The number of these assumptions then represents a measure of how far we are from obtaining a splitting of ABA knowledge base. 
	The concept of a splitting set is then generalised accordingly as follows:
	
	\begin{definition}
		For any ABAF $D=(\lit, \rules, \asm,\contraryempty)$, a set $S\subseteq \lit$ is called a \emph{quasi-splitting} of $D$ if $S=atom(S)$ and for all $r\in \rules$, $head(r)\in S$  implies $body(r)\setminus \asm \subseteq S$.
		Let $V^{\leftarrow}_S=\{b\in \asm\setminus S \mid \exists r,r'\in \rules: b \in body(r)\cap \asm, head(r)\in S, head(r')=\contrary{b}, r\neq r' \}$. We call $S$: 
		\begin{itemize}
			\item $k$-splitting of $D$, if $|V^{\leftarrow}_S|=k$;
			\item (proper) splitting of $D$, if $|V^{\leftarrow}_S|=0$.
		\end{itemize}
	\end{definition}
	As before, the rule-set is split into a bottom and top part, depending on the rule-head respectively being or not in $S$. As a result, $V^{\leftarrow}_S$ is the set of assumptions in the bottom whose contrary is derived in the top. We call $V^{\leftarrow}_S$ the set of \textit{vulnerabilities} with respect to $S$, since it contains assumptions that are attacked by $S$. 
	Whenever $|V^{\leftarrow}_S|\neq 0$, there are some heads in $\rules_2$ whose corresponding assumption may appear in bodies of $\rules_1$. Therefore, the notion of splitting of Definition \ref{def: spl aba} corresponds to a 0-splitting.  
	To account for elements of $V^{\leftarrow}_S$, the ABAFs $D_1$ and $D_2$ induced by the chosen splitting set are constructed in a slightly different way than before. In particular, we fix $D_1$ and $D_2$ as before, but let $\lit_1=S \cup V^{\leftarrow}_S \cup \contrary{V^{\leftarrow}_S}$. 
	Moreover, since contraries in $\contrary{V^{\leftarrow}_S}$ may be derived by top rules, the status of their corresponding assumptions in the bottom depends on rules in the top. Consequently $D_1$ cannot be evaluated in complete isolation from the rest, in contrast with proper splitting. 

	For computing extension of the sub-framework $D_1$, we first need to modify the ABAF. First, we modify the rules by removing body-atoms not in $\lit_1$. Indeed, these atoms occur in $\lit_2$ and are unattacked in $D$, therefore they can be disregarded when evaluating $D_1$. Further, we proceed by adding: (i) a fresh assumption $b'$ (and its contrary $\contrary{b'}$) for each $b\in V^{\leftarrow}_S$; (ii) rules that encode the choice for or against the presence of each assumption $b\in V^{\leftarrow}_S$ in the extension. In this way, we store at the object level the meta-information regarding our choices on each $b\in V^{\leftarrow}_S$. 
	
	\begin{definition}\label{def: [D_1]}
		Let $D=(\lit,\rules,\asm,\contraryempty)$ be an ABAF, $S\subseteq \lit$ be a quasi-splitting of $D$ inducing the sub-frameworks $D_1$ and $D_2$. Moreover, let $V^{\leftarrow}_S$ be the set of vulnerabilities of $D_1$ with respect to $S$ and $(\rules_1)_{\downarrow \lit_1}=\{head(r)\gets body(r)\cap \lit_1\mid r\in \rules_1 \}$. From $D_1$ we construct the ABAF $\llcorner D_1\lrcorner=(\llcorner\lit_1\lrcorner,\llcorner\rules_1\lrcorner,\llcorner\asm_1\lrcorner,\contraryempty)$ by letting:
		\begin{itemize}
			\item $\llcorner\lit_1\lrcorner= \lit_1 \cup \{b',\contrary{b'}\mid b\in V^{\leftarrow}_S\}$;
			\item $\llcorner\rules_1\lrcorner = (\rules_1)_{\downarrow \lit_1} \cup \{\contrary{b}\gets b', \contrary{b'}\gets b\mid b\in V^{\leftarrow}_S\}$. 
		\end{itemize}
	\end{definition}
	Intuitively, the additional rules allow us to choose whether we want to accept an extension $E$ of $\llcorner D_1\lrcorner$ containing $b$ or one that does not. After this choice, we can safely compute the $E$-reduct of $D_2$, as for proper splitting. In this way, we propagate the meta-information to which we committed by means of our choice. A further modification of $D_2$ is now needed to ensure our hypothesis regarding $b$: we add a fact-rule $b\gets$ or a loop-rule $\contrary{b}\gets b$, depending on whether the previously chosen extension $E$ contains $b$ or $b'$. These represent a type of (positive and negative) constraints in ABA. 
	
	\begin{definition}
		Let $D=(\lit,\rules,\asm,\contraryempty)$ be an ABAF, $S$ a quasi-splitting of $D$ into $D_1$ and $D_2$. Moreover, let $V^{\leftarrow}_S$ be the set of vulnerabilities with respect to $S$ and $D^{E}_2$ the $E$-reduct of $D_2$ for some $E\in \sigma(\llcorner D_1\lrcorner)$. We denote with $\ulcorner D^E_2\urcorner=(\lit_2, \ulcorner \rules^E_2 \urcorner,\asm_2,\contraryempty^2)$ the ABAF such that:
		$$\ulcorner \rules^E_2 \urcorner=\rules^E_2 \cup \{b\gets \mid b\in E\cap V^{\leftarrow}_S\}\cup \{\contrary{b}\gets b \mid b'\in E\}.$$ 
	\end{definition}
	
	Notice that such a modification can make $\ulcorner D^E_2\urcorner$ non-flat, as $cl(\emptyset)=\{b\mid b\in E\cap V^{\leftarrow}_S\}$. For stable semantics, however, this does not result in a higher complexity for the same reasoning tasks~\cite{CyrasHT21}.  
	
	\begin{example}
		Consider the ABAF $D=(\lit,\rules,\asm,\contraryempty)$ where $\asm=\{a,b,c,d\}$, $\lit=\asm \cup \contrary{\asm} \cup \{p\}$, and rule-set $\rules$: 
		\begin{align*}
			\contrary{b} & \gets a && \contrary{d} \gets b && \contrary{a} \gets p,c && p\gets b 
		\end{align*}
		First, $E=\{b,c\}$ and $E'=\{a,c,d\}$ are stable extensions in $D$. 
		Now let $S=\{a,\contrary{a},d,\contrary{d},p\}$ be a quasi-splitting of $D$ and $V^{\leftarrow}_S=\{b\}$ the set of vulnerabilities w.r.t.\ $S$. We get $\llcorner \lit_1 \lrcorner=S\cup \{b\} \cup \{\contrary{b}\} \cup \{b',\contrary{b'}\}$ and $\llcorner \rules_1 \lrcorner$ such that:
		\begin{align*}
			\contrary{d} \gets b && \contrary{a} \gets p,c && p\gets b && \contrary{b'} \gets b && \contrary{b} \gets b' 
		\end{align*} 
		
		We derive two stable extensions $E_1=\{b\}$ and $E'_1=\{b',a,d\}$. 
		Now consider $D_2$ with $\lit_2=\lit\setminus S=\{b,c,\contrary{b},\contrary{c}\}$. For the former we get $\ulcorner D^{E_1}_2\urcorner$ with $\ulcorner \rules^{E_1}_2\urcorner=\emptyset \cup \{b\gets \}$ from which we derive $E_2=\{b,c\}$ as a stable extension. For the latter we get $\ulcorner D^{E'_1}_2\urcorner$ with $\ulcorner \rules^{E_1}_2\urcorner=\{\contrary{b} \gets \} \cup \{\contrary{b}\gets b \}$ from which we derive $E'_2=\{c\}$ as a stable extension. We then obtain $E=(E_1\cap S)\cup E_2$ and $E'=(E'_1\cap S)\cup E'_2$. 
		
	\end{example}

	\begin{theorem}
		Let $D=(\lit,\rules,\asm,\contraryempty)$ be an ABAF and $S\subseteq \lit$ a quasi-splitting inducing the sub-frameworks $D_1$ and $D_2$.
		\begin{enumerate}
			\item If $E_1\in \stb(\llcorner D_1\lrcorner)$ and $E_2\in \stb(\ulcorner D^{E_1}_2\urcorner)$, then $(E_1\cap S)\cup E_2\in \stb(D)$.
			\item If $E\in \stb(D)$, then there is a set $X\subseteq \{a'\mid a\in V_S^{\leftarrow}\}$ s.t. $E_1=(E\cap S)\cup X \in \stb(\llcorner D_1\lrcorner)$ and $E_2=E\cap \asm_2 \in \stb(\ulcorner D^{E_1}_2\urcorner)$. 
		\end{enumerate}
	\end{theorem}
	\begin{proof}
		In what follows, let $D'_2 = (\lit_2,\rules'_2,\asm_2,\contraryempty^2)$ be the reduct of $D_2$ w.r.t. $E_1 = (E \cap S)\cup X$ and $E = E_1 \cup E_2$.

		(1.) To prove the statement we need to show $(E_1 \cap S)\cup E_2\in \cf(D)$ and $((E_1\cap S)\cup E_2)^\oplus_{\rules}=\asm$. We start with conflict-freeness. Since $E_1\in \cf(\llcorner D_1\lrcorner)$, then $E_1\cap S\in \cf(\llcorner D_1\lrcorner)$ (less assumptions) and $E_1\cap S\in \cf(D_1)$ (less attacks). 
        We show that  $E_1\cap S\in \cf(D)$. Assume towards contradiction that $E_1\cap S\vdash^R \contrary{a}$ for some $a\in E_1\cap S$ and $R\subseteq \rules$. Since $E_1\cap S\in \cf(D_1)$ and $\rules_2=\{r\in \rules\mid head(r)\notin S\}$, it must be that $E_1\cap S\vdash^{R_2} p$ where
        $R_2=R\cap \rules_2$ for some $p\in body(r')$ and rule $r'\in R_1=R\cap \rules_1$. Since $p\in \lit_2$, we distinguish two cases: if $p\in \contrary{V^{\leftarrow}_S}$, then it is derivable from some $b'$ in $\llcorner D_1\lrcorner$, and $r'$ fires; if $p\notin \contrary{V^{\leftarrow}_S}$, then it gets removed from the body of $r'$ in $\llcorner D_1\lrcorner$, which in turn fires. Thus, in both cases we have that $E_1\cap S\vdash^{R_1'} \contrary{a}$ where $R_1'\subseteq \llcorner\rules_1\lrcorner$. Contradiction. 
        Consider now $E_2\in \stb(\ulcorner D'_2\urcorner)$. Since being stable implies conflict-freeness we immediately get $E_2\in \cf(\ulcorner D'_2\urcorner)$. Again, since $\rules'_2\subseteq \ulcorner \rules'_2\urcorner$, we obtain  $E_2\in \cf(D'_2)$. Furthermore, Proposition \ref{pro:cf-ABA} for proper splittings, together with $E_1\cap S\in \cf(D_1)$ and $E_2\in \cf(D'_2)$, entail $(E_1 \cap S) \cup E_2 \not \vdash^R\contrary{a}$ for any $a\in E_2$ and $R\subseteq \rules$. It only remains to consider possible attacks from $E_2$ to $E_1\cap S$ in $D$. Suppose that there are $T\subseteq E_2$ and $a\in E_1\cap S$ such that $T\vdash^R\contrary{a}$ for some $R\subseteq \rules$. First, notice that since $T\subseteq E_2$, we get $body(R)\cap S=\emptyset\subseteq Th_{D_1}(E_1)$, and thus $R\subseteq \rules'_2$. Moreover, $a\in V^{\leftarrow}_S$ so that $\ulcorner \rules'_2\urcorner = \rules'_2 \cup \{a\gets\}$. Therefore, $T\vdash^R a$ and $T\vdash^R \contrary{a}$ for some $R\subseteq \ulcorner \rules'_2\urcorner$, i.e. $E_2$ is either not conflict-free or not closed in $\ulcorner D'_2\urcorner$. 
		We now show that $((E_1\cap S)\cup E_2)^\oplus_{\rules}=\asm$. Towards contradiction, consider an assumption $a\notin ((E_1\cap S)\cup E_2)^\oplus_{\rules}$. 
		Assume $a\in S$. By hypothesis, $E_1\in \stb(\llcorner D_1\lrcorner)$, i.e. either $a\in E_1$ or $E_1\vdash^{R}\contrary{a}$ for some $R\subseteq \llcorner \rules_1\lrcorner$. From our assumption, we get $a\notin (E_1\cap S)^\oplus_{\rules}$, that is (i) $a\notin E_1\cap S$ and (ii) $E_1\cap S\not\vdash^R\contrary{a}$ for any $R\subseteq\rules$. If (i) holds, we immediately derive $E_1\vdash^{R}\contrary{a}$ for some $R\subseteq \llcorner \rules_1\lrcorner$. Consider now our assumption (ii). Because $a\in S$ we know that every rule of $R$ is contained in $\rules_1$. For the same reason such rules are in $\llcorner\rules_1 \lrcorner$ ($b\notin V^{\leftarrow}_S$). Therefore, $E_1\cap S\not\vdash^R\contrary{a}$ for any $R\subseteq \llcorner \rules_1\lrcorner$ in contradiction with our hypothesis. 
		Assume now $a\in \asm\setminus S$. By hypothesis we know either $a\in E_2$ or $E_2\vdash^{R}\contrary{a}$ for some $R\subseteq \ulcorner \rules'_2\urcorner$. From the assumption, we get $a\notin (E_2)^\oplus_{\rules}$, that is (i) $a\notin E_2$ and (ii) $E_2\not\vdash^{R}\contrary{a}$ for any $R\subseteq \rules$. As before, from (i) and our hypothesis we derive $E_2\vdash^{R}\contrary{a}$ must hold for some $R\subseteq \ulcorner \rules'_2\urcorner$. If $a\in V^{\leftarrow}_S$, there are two possibilities: $a\in E_1\setminus S$ or $a\notin E_1\setminus S$. In the first scenario, $\ulcorner \rules'_2\urcorner = \rules'_2 \cup \{a\gets\}$. Again, $E_2\vdash^R a$ and $E_2\vdash^R \contrary{a}$ for some $R\subseteq \ulcorner \rules'_2\urcorner$, in contradiction with the fact that $E_2$ is a stable extension of $\ulcorner D'_2\urcorner$. If $a\notin E_1\setminus S$, then $a'\in E_1$, which means $\ulcorner\rules'_2\urcorner = \rules'_2 \cup \{\contrary{a}\gets a\}$. Since $a\notin E_2$, the loop-rule $\contrary{a}\gets a$ is not in $R$, therefore $R\subseteq \rules'_2$. Thus, for each rule $r'\in \rules'_2$ there is a rule $r\in \rules_2$ such that $body(r)\subseteq body(r)\cup Th_{D_1}(E_1)$. Hence, it follows directly that $(E_1\cap S)\cup E_2\vdash^R\contrary{a}$ for some $R\subseteq \rules_1\cup \rules_2=\rules$. If $a\notin V^{\leftarrow}_S$, then $a\notin \asm_1$. If $E_2\vdash^R\contrary{a}$ for some $R\subseteq \ulcorner \rules'_2\urcorner$, it is not because $\{\contrary{a}\gets a\}\subseteq \ulcorner \rules'_2\urcorner$. Thus  $R\subseteq \rules'_2$. As before, for each rule $r'\in \rules'_2$ there is exactly one rule $r\in \rules_2$ such that $body(r)\subseteq body(r')\cup Th_{D_1}(E_1 \cap S)$. As a result, in $D$ it holds that $(E_1\cap S)\cup E_2\vdash \contrary{a}$. Contradiction. 
		
		(2.) 
		First we get $E\in \cf(D)$ and thus $E\cap S\in \cf(D_1)$ (less attacks). Now let $B=\asm_1\setminus(E\cap S)^\oplus_{\rules}$. Since $E\in \stb(D)$, it attacks every other assumption. Hence, we can infer that assumptions in $B$ are contained in $\asm_1$ and attacked by $E\cap \asm_2$ in $D$, that is $B\subseteq E^+_{\rules}\setminus
		(E\cap S)^\oplus_{\rules}=(E\cap \asm_2)^+_{\rules}$.  Therefore there is a rule $r\in \rules_2$ with  $head(r)=\contrary{b}$ for each $b\in  B$, meaning that $B= V^{\leftarrow}_S$. Now let $X=\{b'\mid b\in B\}$. Thus $\ulcorner \rules'_2\urcorner$ contains a pair of rule $\{\contrary{b}\gets b', \contrary{b'}\gets b\}$ for each $b\in B$. Consequently, conflict-freeness of $(E\cap S)\cup X$ is ensured since $b\notin E\cap S$ for all $b\in B$. Moreover, $X$ attacks every $b\in B$ in $\llcorner D_1\lrcorner$, making $(E\cap S)\cup X$ stable.    
		It now remains to show $E_2=E\cap \asm_2 \in \stb(\ulcorner D'_2\urcorner)$. As before, we know that $E\cap \asm_2 \in \cf(D_2)$ since $E$ is conflict-free in $D$ (less assumptions), and $E\cap \asm_2 \in \cf(D'_2)$ because $Th_{D'_2}(E\cap \asm_2)\subseteq Th_{D_2}(E\cap \asm_2)$ (less rules and attacks). 
		Consider now the modified framework $\ulcorner D'_2\urcorner$ wrt $(E\cap S)\cup X$. By construction, $E\cap \asm_2 \notin \cf(\ulcorner D'_2\urcorner)$ only if $b\in B \cap (E\cap \asm_2)$. Recall that $B\subseteq (E\cap \asm_2)^+_{\rules}$. Thus, $E\cap \asm_2 \notin \cf(D)$. By contradiction, we derive that $E\cap \asm_2$ is conflict free in $\ulcorner D'_2\urcorner$.  
		We now show that $E_2\vdash^{R} \contrary{a}$ for all $a\in \asm_2\setminus E_2$ and some $R\subseteq \ulcorner \rules'_2\urcorner$. Towards contradiction, we assume there is an $a\in \asm_2\setminus E_2$ such that $E_2\not\vdash^{R} \contrary{a}$, i.e. $\contrary{a}\notin Th_{\ulcorner D'_2\urcorner}(E_2)$. Therefore, since $a\notin E_2$, we get $\contrary{a}\notin Th_{D'_2}(E_2)$. Hence, before the reduct is applied, it holds that $(E\cap S)\cup E_2\not\vdash^R \contrary{a}$ with $R\subseteq \rules_2$. Since no rule $r\in \rules_1$ is such that $head(r)=\contrary{a}$, we derive $(E\cap S)\cup E_2\not\vdash^R \contrary{a}$ in $D$, contradicting our hypothesis. 
		Finally, we ensure that $cl(E_2)=E_2$ in $\ulcorner D'_2\urcorner$. Assume the contrary holds. Since $D'_2$ is flat, that means $\{a\gets\}\subseteq \ulcorner \rules'_2\urcorner$ and $a\notin E_2$. These facts respectively entail $a\in E_1$ and $E_2\vdash \contrary{a}$, against the conflict-freeness of $E$. Thus, $E_2$ is conflict-free, closed and attacks every other assumption. 
	\end{proof}

    \section{Computing the Splitting Set}\label{sec:computeSplitting}
    In order to apply the introduced splitting schemes, we need to first compute a splitting of the ABAF and SETAF respectively. In this section, we show that this can be reduced to the same graph problem as for computing a splitting of an AFs, and we can therefore reuse existing algorithms~\cite{BaumannBW11} for the computation of our splittings.  
    First, for a SETAF $SF$ we consider its primal graph, which draws an egde from argument $a$ to $b$
    if there is an attack $(T,b)\in SF$ with $a \in T$.
    For ABA splittings, we build a Dependency Graph $G_D$ of the ABAF to depict the influences among the sentences in $\mathcal{L}$~\cite{BlumelRTT25,NMRRapberger0W22}. 
	That is, for an ABAF $D=(\lit,\rules,\asm,\contraryempty)$ the dependency graph $G_D=(V_D,E_D)$ is constructed as follows: we take a node $a\in V_D$ for each atom in $a\in \mathcal{L}$; for every rule $r\in \rules$ we take: 
    \begin{itemize}
        \item an edge $(a,b)\in E_D$ iff $head(r)=b$ and $a\in body(r)$,
        \item two symmetric edges $(a,b)$ and $(b,a)$ between each assumption node $a\in \asm$ and its contrary node $b=\contrary{a}$. 
    \end{itemize}
    The first condition makes sure that influences among sentences encoded via rules are mirrored by the edges. Moreover, the second condition ensures that $S=atom(S)$ for any splitting set $S$ when we later compute the SCCs.

    \begin{example}\label{ex:dependency}
    Recall the ABAF $D$ of Example~\ref{ex: big ABAF}. Its corresponding dependency graph $G_D=(V_D,E_D)$ is: 

    
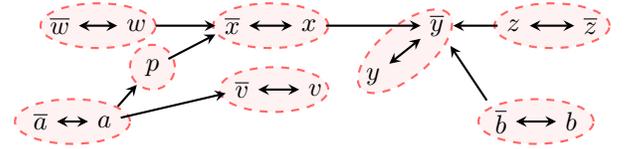
\begin{center}
\begin{tikzpicture}[scale=0.85,>=stealth]
 \filldraw[color=red!60,fill=red!5,dashed,thick] (-0.5,-0.5) ellipse (0.9 and 0.35);
 \filldraw[color=red!60,fill=red!5,dashed,thick] (6.75,-0.5) ellipse (0.9 and 0.35);
 \filldraw[color=red!60,fill=red!5,dashed,thick] (2.65,0) ellipse (0.85 and 0.35);
 \filldraw[color=red!60,fill=red!5,dashed,thick] (-0.1,1) ellipse (0.9 and 0.35);
 \filldraw[color=red!60,fill=red!5,dashed,thick] (2.6,1) ellipse (0.9 and 0.35);
 \filldraw[color=red!60,fill=red!5,dashed,thick,rotate around={40:(4.7,0.6)}] (4.7,0.6) ellipse (0.9 and 0.4);
 \filldraw[color=red!60,fill=red!5,dashed,thick] (7,1) ellipse (0.9 and 0.35);
 \filldraw[color=red!60,fill=red!5,dashed,thick] (0.75,0.35) circle (0.35); 
 
	\path
	(0,-0.5) node (a){$a$}
    (-1,-0.5) node (a'){$\contrary{a}$}
	(7.3,-0.5) node(b){$b$}
    (6.2,-0.5) node(b'){$\contrary{b}$}
	(3.3,0) node(v){$v$}
    (2.15,0) node(v'){$\contrary{v}$}
	(.5,1) node (w) {$w$}
    (-.7,1) node (w') {$\contrary{w}$}
	(3.2,1) node (x) {$x$}
    (2,1) node (x') {$\contrary{x}$}
	(4.2,0.2) node (y) {$y$}
    (5.2,1) node (y') {$\contrary{y}$}
	(6.4,1) node (z) {$z$}
    (7.6,1) node (z') {$\contrary{z}$}
    (0.75,0.35) node (p) {$p$}
	;
	
	\path[->,thick]
	(a) edge[<->] (a')
	(b) edge[<->] (b')
    (v) edge[<->] (v')
    (w) edge[<->] (w')
    (x) edge[<->] (x')
    (y) edge[<->] (y')
    (z) edge[<->] (z')
    (x) edge (y')
	(b') edge (y')
	(z) edge (y')
	(w) edge (x')
	(a) edge (p)
    (p) edge (x')
    (a) edge (v')
	;		
   
\end{tikzpicture}
    \end{center} 
\end{example}
    We now have a simple directed graph, whose splittings correspond to the splittings of the ABAF, SETAF respectively. 
    Thus, we can apply the strategy of~\citeauthor{BaumannBW11}~(\citeyear{BaumannBW11}) to obtain a splitting set. 
	In particular, after running Tarjan's algorithm on $G_D$ to compute the SCCs, we contract every SCC to a single vertex 
    leading to an acyclic graph $G^{\circ}_D$, where nodes are the SCCs (highlighted in Example~\ref{ex:dependency}) of the original dependency graph. 

    \begin{example}\label{ex:dependency2}
    We continue Example~\ref{ex:dependency} and construct $G^{\circ}_D$. 
    
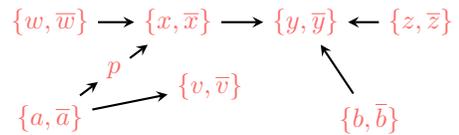
    \begin{center}
\begin{tikzpicture}[scale=0.85,>=stealth]
	\path
	(2,-0.5) node[color=red!60]  (a) {$\{a,\contrary{a}\}$}
    (7,-0.5) node[color=red!60] (b){$\{b,\contrary{b}\}$}
    (4.5,0) node[color=red!60] (v){$\{v,\contrary{v}\}$}
	(2,1) node[color=red!60] (w) {$\{w,\contrary{w}\}$}
    (4,1) node[color=red!60] (x) {$\{x,\contrary{x}\}$}
    (6,1) node[color=red!60] (y) {$\{y,\contrary{y}\}$}
	(7.8,1) node[color=red!60] (z) {$\{z,\contrary{z}\}$}
    (3,0.25) node[color=red!60] (p) {$p$}
	;
	
	\path[->,thick]
    (x) edge (y)
	(b) edge (y)
	(z) edge (y)
	(w) edge (x)
	(a) edge (p)
    (p) edge (x)
    (a) edge (v)
	;		
   
\end{tikzpicture}
    \end{center} 
\end{example}
    As a result, each splitting set of $D$ corresponds to exactly one splitting of $G^{\circ}_D$. 
    Further,~\citeauthor{BaumannBW11}~(\citeyear{BaumannBW11}) focus on splitting that divides $G^{\circ}_D$ as evenly as possible in two parts with respect to the total number of nodes~\cite[Section 3.2, Algorithm 2]{BaumannBW11}. In this sense, a desirable splittings for $G^{\circ}_D$ in Example~\ref{ex:dependency2} could be $S=\{a,\contrary{a},b,\contrary{b},p,v,\contrary{v}\}$. 

	Notice that the size of the dependency graph $G_D$ is linear in the size of the ABAF (as $|V_D|=|\lit|$) and the SCCs can be computed in linear time~\cite{Tarjan72}. 
    In addition, the present approach can also be employed to compute quasi-splittings. Once the dependency graph has been obtained, we contract only those pairs of nodes comprising an assumption and its contrary. Subsequently, we can apply the same techniques used for Dung-style AFs~\cite{BaumannBDW12}, although these are more involved than those required for standard splitting. This strategy relies on the computation of minimum cuts in directed graphs using existing polynomial-time algorithms~\cite{HAO1994424}.

	\section{Relating ABA and SETAF Splittings}\label{sec:comparison}
	In this section, we compare splitting schemes for SETAFs and ABAFs.
	First, notice that for an ABAF $D$, every splitting set yields a splitting on the corresponding SETAF $SF_D$. Conversely, several ABA splitting sets may correspond to the same SETAF splitting. 
	
	\begin{observation}
		Let $D=(\lit,\rules,\asm,\contraryempty)$ be an ABAF and $SF_D=(A_D,R_D)$ the corresponding SETAF. Further, let $S$ be a splitting set of $D$ into $D_1$ and $D_2$. Then there is a splitting $(SF_{1},SF_{2},R_3)$ of $SF_D$ such that the assumptions in $S$ are exactly the arguments of $SF_1$. 
		Vice versa, for a SETAF $SF$ and a splitting $(SF_{1},SF_{2},R_3)$, there are several splitting sets $S_1,\dots, S_n$ for the ABAF $D_{SF}$ such that the arguments of $SF_1$ coincide
        with the assumptions in $S_i$.
	\end{observation}   
	
	This opens the possibility of applying splitting on ABAFs before or after the instantiation into the corresponding SETAFs. To illustrate this point, consider the following.
	\begin{example}
    Consider the ABAF $D =(\lit, \rules, \asm,\contraryempty)$ (left) and its corresponding SETAF $SF_D$ (right).

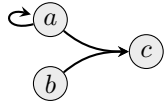
\begin{minipage}{.29\textwidth}
	\hspace{-14pt}
	\begin{tabular}{cl}
		$\asm=$ &$\{a,b,c\}$\\
		$\lit=$ &$\asm\cup \contrary{\asm}\cup \{p,q,s\}$\\
		$\rules=$&$\{\contrary{c} \gets s,q, \;\ s\gets b, $ \\
		& \;\ $q\gets p, \;\  p\gets a, \;\ \contrary{a}\gets a\}$ 
	\end{tabular}
\end{minipage}
\begin{minipage}{.15\textwidth}
\begin{tikzpicture}[scale=0.85,>=stealth]
	\path
	(1.5,0.5) node[arg] (a){$a$}
	(1.5,-0.5) node[arg](b){$b$}
	(3,0) node[arg](c){$c$}
	;
	
	\path[->,>=stealth,thick]
	(a) edge[loop left] (a)
	(a) edge[out=-40,in=180] (c)
	(b) edge[out=40,in=180] (c)
	;		
\end{tikzpicture}
\end{minipage}
		 The corresponding SETAF is $SF_D=(A_D,R_D)$ where $A_D=\asm$ and $R_D=\{(a,a),(\{a,b\},c)\}$. By appropriately combining $\{a,\contrary{a},b,\contrary{b}\}$ with sentences in $\{p,q,s\}$ we get multiple splitting sets of $D$ corresponding to a single SETAF splitting with $R_3=\{(\{a,b\},c)\}$.
         
         \begin{minipage}{.29\textwidth}
	\hspace{-14pt}
	\begin{tabular}{cl}
		$S_1=$ &$\{a,b,\contrary{a},\contrary{b}\}$\\
        $S_2=$ &$S_1 \cup \{p\}$\\
        $S_3=$ &$S_1 \cup \{s\}$\\
        $S_4=$ &$S_1 \cup \{p,q,s\}$
	\end{tabular}
\end{minipage}
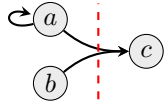
\begin{minipage}{.15\textwidth}
\begin{tikzpicture}[scale=0.85,>=stealth]
	\path
	(1.5,0.5) node[arg] (a){$a$}
	(1.5,-0.5) node[arg](b){$b$}
	(3,0) node[arg](c){$c$}
	;
	
	\path[->,>=stealth,thick]
	(a) edge[loop left] (a)
	(a) edge[out=-40,in=180] (c)
	(b) edge[out=40,in=180] (c)
	;		
    \draw [thick,dashed,red] (2.25,-0.75) -- (2.25,0.75);
\end{tikzpicture}
\end{minipage}
         Moreover, selecting $\{b\}$ as preferred extension of $D_1$ and $SF_1$, we add, respectively, the fresh assumption $x_u$ attacking itself and $c$, and the self-attack $(c,c)$ to $SF_2$. This results in $\emptyset$ being the only preferred extension of the $D_2^{\star}$ and $SF_2^{\star}$.  
	\end{example}
	As the previous example showcases, both strategies allow for the same splittings and, 
    indeed, compute the same expected result.  
    When combining splitting with an instantiation based solver, the above considerations open different possibilities of how and when to split throughout the reasoning process. A first strategy, when instantiation is feasible, is to generate the SETAF and apply the splitting schema at the level of the graph. 
	Alternatively, when this is not possible, one can first split the ABAF $D$ into $D_1$ and $D_2$, then instantiate them into SETAFs $SF_{D_1}$ and $SF_{D_2}$ and apply the SETAF splitting algorithm, i.e.\ modify $SF_{D_2}$ into $(SF_{D_2})^\star$ and compute extensions to be merged.    
	This might allow one to utilise the SETAF splitting algorithm in cases where the argument graph of the entire ABAF is too hard to compute.    
	
	\section{Related Work}\label{sec:related}
	We hereby clarify the relation of our splitting approach to SCC-recursiveness~\cite{DvorakKUW24,BlumelRTT25}, and we compare our notion of splitting to existing ones for logic programs (LPs)~\cite{LifschitzT94} and abstract dialectical frameworks (ADFs)~\cite{Linsbichler14}, \textcolor{blue} and syntax splitting~\cite{parikh99}. 
	
	In the incremental computation approach induced by the SCC-recursive property, one
	computes the extensions in subframeworks, and ultimately combines the
	thereby computed extension parts (as in the splitting approach).
	In contrast to splittings however, this is restricted to subframeworks that make up \emph{strongly connected components} w.r.t.\ the primal graph of the SETAF or the ABAF's dependency graph~\cite{DvorakKUW24,BlumelRTT25}. Indeed, given the primal graph of a SETAF (resp. an ABAF's dependency graph), our approach corresponds to splitting between arguments (resp. sentences) such that the involved attacks (positive edges) have the same direction. 
	However, splitting is more general in this regard, as the subframeworks do not have to be SCCs.
	Finally, SCC-recursiveness relies on a generalised semantics to deal with the decisions of prior parts of the framework, in contrast to the syntactic manipulation-based approach of splitting.
	This allows splitting to be implemented on top of existing ABA and SETAF solvers. 
	
	It has been shown that ABA captures normal logic programs with negation-as-failure under several semantics~\cite{CaminadaS17}. In particular, partial stable, well-founded and regular models in logic programming correspond to complete, grounded and preferred extensions in ABA. The translation works by turning every negative (resp. positive) literals of an LP into an assumption (resp. contrary), and vice versa. By means of this simple transformation, our splitting algorithm can be adapted in the context of logic programming to obtain a divide-and-conquer approach beyond stable semantics, thus generalising the original result~\cite{LifschitzT94}. As a by-product, we are then able to split normal logic programs under semantics based on partially stable, well-founded and regular models. 
	
	ADFs~\cite{BrewkaESWW18} are an expressive argumentation formalism, where each argument is associated with a propositional formula over arguments as variables as an {acceptance condition}.
	It is well-known that SETAFs can be interpreted as a special kind of ADFs with acceptance conditions in the form of a conjunction of disjunctive clauses of negated literals~\cite{DvorakZW23}. That is, in principle we can apply ADF splitting to SETAFs. 
	However, it is not clear that following the ADF approach the modified second framework again is of the desired (SETAF-like) form and whether one can avoid certain overheads in the simpler case of SETAFs. Upon closer inspection and with minor syntactic manipulation the ADF approach in the special case of SETAF-like frameworks is similar to introducing an artificial self-attacking argument that behaves similarly to Definition~\ref{def: aba-modif}. However, such a trick is not needed for SETAFs in our case. 
    
    \citeauthor{parikh99}~(\citeyear{parikh99}) introduced the concept of \emph{syntax splitting} to characterize belief sets that contain independent pieces of information. In this setting, a syntax splitting consists of a partition of the knowledge base into independent components, where each independent piece of information can be represented using only one part of these partition. In this paper, we have considered a notion of splitting which is more general, where the different parts of the knowledge base may not be independent from each other. For this, we provided suitable modifications to account for such interdependencies.  
    	
	\section{Conclusion and Future Work}\label{sec:conc}
	
	In this paper, we have presented a modification-based approach to splitting assumption-based argumentation frameworks. First, we have introduced a splitting schema for SETAFs, to make splitting available on ABAFs after instantiation. Leveraging on the close connection between SETAFs and ABAFs, we have thus proposed a way to split when reasoning in ABA is performed indirectly via the corresponding argument graph. 
	Moreover, to overcome the requirement of an instantiation and its associated costs, we have introduced a  splitting schema that works directly on ABA knowledge bases.
	For both approaches, we have shown that extensions of a given ABAF can be obtained incrementally from its sub-frameworks, by means of simple syntactic modifications. Conversely, we can project an arbitrary extension of the whole framework to its sub-frameworks.
	Since this is bound to the specific structure of the underlying ABAF, we have considered a more general variant of splitting called parametrised splitting inspired by~\citeauthor{BaumannBDW12}~(\citeyear{BaumannBDW12}). Moreover, it is easy to see that each of the steps involved can be carried out efficiently and implemented on top of common ABA (or SETAF) solvers.  
    Indeed, the splitting techniques introduced in the paper require only syntactic modifications to the sub-frameworks, which do not lead to an exponential increase in size (and may, in some cases, even reduce it). In particular, through the reduct and modification, we introduce at most one additional rule and one additional assumption with respect to the original sub-framework. In the case of parametrised splitting, the number of additional rules and assumptions is bounded by $2k$ for a $k$-splitting. 
    Therefore, an obvious next step is to implement our algorithm and perform an experimental evaluation in the spirit of~\citeauthor{BaumannBW11}~(\citeyear{BaumannBW11})\textcolor{blue}, where executions with splitting result in an average acceleration of 60\% in comparison to executions without splitting. In particular, we believe that parametrised splitting could be helpful in the context of the recently proposed Argumentative Causal Discovery~\cite{RussoRT24}.
    In fact, this framework faces a major challenge in terms of its scalability, as it exhibits suboptimal performance on larger instances.
	
	\section*{Acknowledgments}
    This paper is an extended and revised version of two workshop papers that were presented in SAFA 2024~\cite{BuraglioDKW24} and NMR 2025~\cite{Buraglio25}.  
    We are grateful to the anonymous reviewers for their valuable comments and suggestions on a preliminary version of this paper. Moreover, the authors would like to thank Matthias K\"{o}nig for his contributions to the paper presented at SAFA 2024. 
	This work has been supported by the Austrian Science Fund (FWF) under grant 10.55776/COE122 and by the European Union's Horizon 2020 research and innovation programme (under grant agreement 101034440). 
	\bibliographystyle{kr}
	\bibliography{sample}

@inproceedings{Buraglio25,
  author       = {Giovanni Buraglio},
  editor       = {Anna Rapberger and
                  Sebastian Rudolph},
  title        = {Splitting Assumption-Based Argumentation Frameworks},
  booktitle    = {Proceedings of the 23rd International Workshop on Non-Monotonic Reasoning
                  {(NMR} 2025) co-located with the 22nd International Conference on
                  Principles of Knowledge Representation and Reasoning {(KR} 2025),
                  Melbourne, Australia, November 11-13, 2025},
  series       = {{CEUR} Workshop Proceedings},
  pages        = {17--31},
  publisher    = {CEUR-WS.org},
  year         = {2025},
  url          = {https://ceur-ws.org/Vol-4071/paper2.pdf},
  bibsource    = {dblp computer science bibliography, https://dblp.org}
}

@inproceedings{BuraglioDKW24,
  author       = {Giovanni Buraglio and
                  Wolfgang Dvor{\'{a}}k and
                  Matthias K{\"{o}}nig and
                  Stefan Woltran},
  editor       = {AnneMarie Borg and
                  Stefan Ellmauthaler and
                  Jean{-}Guy Mailly and
                  Andreas Niskanen},
  title        = {Splitting Argumentation Frameworks with Collective Attacks},
  booktitle    = {Proceedings of the Fifth International Workshop on Systems and Algorithms
                  for Formal Argumentation co-located with 10th International Conference
                  on Computational Models of Argument {(COMMA} 2024), Hagen, Germany,
                  September 17th, 2024},
  series       = {{CEUR} Workshop Proceedings},
  pages        = {41--55},
  publisher    = {CEUR-WS.org},
  year         = {2024},
  url          = {https://ceur-ws.org/Vol-3757/paper3.pdf},
  bibsource    = {dblp computer science bibliography, https://dblp.org}
}

@article{parikh99,
  title={Beliefs, belief revision, and splitting languages},
  author={Parikh, Rohit},
  year = {1999},
  journal = {Logic, Language and Computation},
volume = {2},
pages = {266–278}
}

@article{HAO1994424,
title = {A Faster Algorithm for Finding the Minimum Cut in a Directed Graph},
journal = {Journal of Algorithms},
volume = {17},
number = {3},
pages = {424-446},
year = {1994},
issn = {0196-6774},
doi = {https://doi.org/10.1006/jagm.1994.1043},
url = {https://www.sciencedirect.com/science/article/pii/S0196677484710431},
author = {J.X. Hao and J.B. Orlin}
}

@inproceedings{NMRRapberger0W22,
  author       = {Anna Rapberger and
                  Markus Ulbricht and
                  Johannes Peter Wallner},
  editor       = {Ofer Arieli and
                  Giovanni Casini and
                  Laura Giordano},
  title        = {Argumentation Frameworks Induced by Assumption-Based Argumentation:
                  Relating Size and Complexity},
  booktitle    = {Proceedings of the 20th International Workshop on Non-Monotonic Reasoning,
                  {NMR} 2022, Part of the Federated Logic Conference (FLoC 2022), Haifa,
                  Israel, August 7-9, 2022},
  series       = {{CEUR} Workshop Proceedings},
  volume       = {3197},
  pages        = {92--103},
  publisher    = {CEUR-WS.org},
  year         = {2022},
  url          = {https://ceur-ws.org/Vol-3197/paper9.pdf},
  bibsource    = {dblp computer science bibliography, https://dblp.org}
}

@article{Tarjan72,
  author       = {Robert Endre Tarjan},
  title        = {Depth-First Search and Linear Graph Algorithms},
  journal      = {{SIAM} J. Comput.},
  volume       = {1},
  number       = {2},
  pages        = {146--160},
  year         = {1972},
  url          = {https://doi.org/10.1137/0201010},
  doi          = {10.1137/0201010},
  bibsource    = {dblp computer science bibliography, https://dblp.org}
}

@article{BaroniGL14,
	author       = {Pietro Baroni and
	Massimiliano Giacomin and
	Beishui Liao},
	title        = {On topology-related properties of abstract argumentation semantics.
	{A} correction and extension to Dynamics of argumentation systems:
	{A} division-based method},
	journal      = {Artif. Intell.},
	volume       = {212},
	pages        = {104--115},
	year         = {2014},
	url          = {https://doi.org/10.1016/j.artint.2014.03.003},
	doi          = {10.1016/J.ARTINT.2014.03.003},
	bibsource    = {dblp computer science bibliography, https://dblp.org}
}

@article{Liao13,
	author       = {Beishui Liao},
	title        = {Toward incremental computation of argumentation semantics: {A} decomposition-based
	approach},
	journal      = {Ann. Math. Artif. Intell.},
	volume       = {67},
	number       = {3-4},
	pages        = {319--358},
	year         = {2013},
	url          = {https://doi.org/10.1007/s10472-013-9364-8},
	doi          = {10.1007/S10472-013-9364-8},
	bibsource    = {dblp computer science bibliography, https://dblp.org}
}

@article{LehtonenWJ21a,
	author       = {Tuomo Lehtonen and
	Johannes Peter Wallner and
	Matti J{\"{a}}rvisalo},
	title        = {Declarative Algorithms and Complexity Results for Assumption-Based
	Argumentation},
	journal      = {J. Artif. Intell. Res.},
	volume       = {71},
	pages        = {265--318},
	year         = {2021},
	url          = {https://doi.org/10.1613/jair.1.12479},
	doi          = {10.1613/JAIR.1.12479},
	bibsource    = {dblp computer science bibliography, https://dblp.org}
}

@article{LehtonenWJ21b,
	author       = {Tuomo Lehtonen and
	Johannes Peter Wallner and
	Matti J{\"{a}}rvisalo},
	title        = {Harnessing Incremental Answer Set Solving for Reasoning in Assumption-Based
	Argumentation},
	journal      = {Theory Pract. Log. Program.},
	volume       = {21},
	number       = {6},
	pages        = {717--734},
	year         = {2021},
	url          = {https://doi.org/10.1017/S1471068421000296},
	doi          = {10.1017/S1471068421000296},
	bibsource    = {dblp computer science bibliography, https://dblp.org}
}

@inproceedings{LehtonenRT0W24,
  author       = {Tuomo Lehtonen and
                  Anna Rapberger and
                  Francesca Toni and
                  Markus Ulbricht and
                  Johannes Peter Wallner},
  title        = {Instantiations and Computational Aspects of Non-Flat Assumption-based
                  Argumentation},
  booktitle    = {Proceedings of the Thirty-Third International Joint Conference on
                  Artificial Intelligence, {IJCAI} 2024, Jeju, South Korea, August 3-9,
                  2024},
  pages        = {3457--3465},
  publisher    = {ijcai.org},
  year         = {2024},
  url          = {https://www.ijcai.org/proceedings/2024/383},
  bibsource    = {dblp computer science bibliography, https://dblp.org}
}

@inproceedings{GresslerDW24,
	author       = {Alexander Gre{\ss}ler and
	Wolfgang Dvo\v{r}{\'{a}}k and
	Stefan Woltran},
	editor       = {Chris Reed and
	Matthias Thimm and
	Tjitze Rienstra},
	title        = {The {GSAF} Solver and Verifier},
	booktitle    = {Computational Models of Argument - Proceedings of {COMMA} 2024, Hagen,
	Germany, September 18-20, 2024},
	series       = {Frontiers in Artificial Intelligence and Applications},
	volume       = {388},
	pages        = {353--354},
	publisher    = {{IOS} Press},
	year         = {2024},
	url          = {https://doi.org/10.3233/FAIA240336},
	doi          = {10.3233/FAIA240336},
	bibsource    = {dblp computer science bibliography, https://dblp.org}
}

@inproceedings{DvorakGW18,
  author       = {Wolfgang Dvo\v{r}{\'{a}}k and
                  Alexander Gre{\ss}ler and
                  Stefan Woltran},
  editor       = {Matthias Thimm and
                  Federico Cerutti and
                  Mauro Vallati},
  title        = {Evaluating SETAFs via Answer-Set Programming},
  booktitle    = {Proceedings of the Second International Workshop on Systems and Algorithms
                  for Formal Argumentation {(SAFA} 2018) co-located with the 7th International
                  Conference on Computational Models of Argument {(COMMA} 2018), Warsaw,
                  Poland, September 11, 2018},
  series       = {{CEUR} Workshop Proceedings},
  volume       = {2171},
  pages        = {10--21},
  publisher    = {CEUR-WS.org},
  year         = {2018},
  url          = {https://ceur-ws.org/Vol-2171/paper\_2.pdf},
  bibsource    = {dblp computer science bibliography, https://dblp.org}
}

@inproceedings{DimopoulosDKRUW24,
	author       = {Yannis Dimopoulos and
	Wolfgang Dvo\v{r}{\'{a}}k and
	Matthias K{\"{o}}nig and
	Anna Rapberger and
	Markus Ulbricht and
	Stefan Woltran},
	editor       = {Michael J. Wooldridge and
	Jennifer G. Dy and
	Sriraam Natarajan},
	title        = {Redefining {ABA+} Semantics via Abstract Set-to-Set Attacks},
	booktitle    = {Thirty-Eighth {AAAI} Conference on Artificial Intelligence, {AAAI}
	2024, Thirty-Sixth Conference on Innovative Applications of Artificial
	Intelligence, {IAAI} 2024, Fourteenth Symposium on Educational Advances
	in Artificial Intelligence, {EAAI} 2014, February 20-27, 2024, Vancouver,
	Canada},
	pages        = {10493--10500},
	publisher    = {{AAAI} Press},
	year         = {2024},
	url          = {https://doi.org/10.1609/aaai.v38i9.28918},
	doi          = {10.1609/AAAI.V38I9.28918},
	bibsource    = {dblp computer science bibliography, https://dblp.org}
}

@article{CarreraI15,
	author       = {{\'{A}}lvaro Carrera and
	Carlos Angel Iglesias},
	title        = {A systematic review of argumentation techniques for multi-agent systems
	research},
	journal      = {Artif. Intell. Rev.},
	volume       = {44},
	number       = {4},
	pages        = {509--535},
	year         = {2015},
	url          = {https://doi.org/10.1007/s10462-015-9435-9},
	doi          = {10.1007/S10462-015-9435-9},
	bibsource    = {dblp computer science bibliography, https://dblp.org}
}

@inproceedings{DimopoulosMM19,
	author       = {Yannis Dimopoulos and
	Jean{-}Guy Mailly and
	Pavlos Moraitis},
	editor       = {Edith Elkind and
	Manuela Veloso and
	Noa Agmon and
	Matthew E. Taylor},
	title        = {Argumentation-based Negotiation with Incomplete Opponent Profiles},
	booktitle    = {Proceedings of the 18th International Conference on Autonomous Agents
	and MultiAgent Systems, {AAMAS} '19, Montreal, QC, Canada, May 13-17,
	2019},
	pages        = {1252--1260},
	publisher    = {International Foundation for Autonomous Agents and Multiagent Systems},
	year         = {2019},
	url          = {http://dl.acm.org/citation.cfm?id=3331829},
	bibsource    = {dblp computer science bibliography, https://dblp.org}
}

@inproceedings{FanT12,
	author       = {Xiuyi Fan and
	Francesca Toni},
	editor       = {Luc De Raedt and
	Christian Bessiere and
	Didier Dubois and
	Patrick Doherty and
	Paolo Frasconi and
	Fredrik Heintz and
	Peter J. F. Lucas},
	title        = {Agent Strategies for ABA-based Information-seeking and Inquiry Dialogues},
	booktitle    = {{ECAI} 2012 - 20th European Conference on Artificial Intelligence.
	Including Prestigious Applications of Artificial Intelligence {(PAIS-2012)}
	System Demonstrations Track, Montpellier, France, August 27-31 , 2012},
	series       = {Frontiers in Artificial Intelligence and Applications},
	volume       = {242},
	pages        = {324--329},
	publisher    = {{IOS} Press},
	year         = {2012},
	url          = {https://doi.org/10.3233/978-1-61499-098-7-324},
	doi          = {10.3233/978-1-61499-098-7-324},
	bibsource    = {dblp computer science bibliography, https://dblp.org}
}

@inproceedings{GaoTWX16,
	author       = {Yang Gao and
	Francesca Toni and
	Hao Wang and
	Fanjiang Xu},
	editor       = {Catholijn M. Jonker and
	Stacy Marsella and
	John Thangarajah and
	Karl Tuyls},
	title        = {Argumentation-Based Multi-Agent Decision Making with Privacy Preserved},
	booktitle    = {Proceedings of the 2016 International Conference on Autonomous Agents
	{\&} Multiagent Systems, Singapore, May 9-13, 2016},
	pages        = {1153--1161},
	publisher    = {{ACM}},
	year         = {2016},
	url          = {http://dl.acm.org/citation.cfm?id=2937093},
	bibsource    = {dblp computer science bibliography, https://dblp.org}
}

@article{HadouxHP23,
	author       = {Emmanuel Hadoux and
	Anthony Hunter and
	Sylwia Polberg},
	title        = {Strategic argumentation dialogues for persuasion: Framework and experiments
	based on modelling the beliefs and concerns of the persuadee},
	journal      = {Argument Comput.},
	volume       = {14},
	number       = {2},
	pages        = {109--161},
	year         = {2023},
	url          = {https://doi.org/10.3233/AAC-210005},
	doi          = {10.3233/AAC-210005},
	bibsource    = {dblp computer science bibliography, https://dblp.org}
}

@article{Toni13,
	author       = {Francesca Toni},
	title        = {A generalised framework for dispute derivations in assumption-based
	argumentation},
	journal      = {Artif. Intell.},
	volume       = {195},
	pages        = {1--43},
	year         = {2013},
	url          = {https://doi.org/10.1016/j.artint.2012.09.010},
	doi          = {10.1016/J.ARTINT.2012.09.010},
	bibsource    = {dblp computer science bibliography, https://dblp.org}
}

@article{CaminadaS17,
	author       = {Martin Caminada and
	Claudia Schulz},
	title        = {On the Equivalence between Assumption-Based Argumentation and Logic
	Programming},
	journal      = {J. Artif. Intell. Res.},
	volume       = {60},
	pages        = {779--825},
	year         = {2017},
	url          = {https://doi.org/10.1613/jair.5581},
	doi          = {10.1613/JAIR.5581},
	bibsource    = {dblp computer science bibliography, https://dblp.org}
}

@inproceedings{BlumelRTT25,
	title     = {On Independence and SCC-Recursiveness in Assumption-Based Argumentation},
	author    = {Blümel, Lydia and Rapberger, Anna and Thimm, Matthias and Toni, Francesca},
	booktitle = {Proceedings of the Thirty-Fourth International Joint Conference on
	Artificial Intelligence, {IJCAI-25}},
	publisher = {International Joint Conferences on Artificial Intelligence Organization},
	editor    = {James Kwok},
	pages     = {4382--4390},
	year      = {2025},
	month     = {8},
	doi       = {10.24963/ijcai.2025/488},
	url       = {https://doi.org/10.24963/ijcai.2025/488},
}

@article{CyrasHT21,
	author       = {Kristijonas Cyras and
	Quentin Heinrich and
	Francesca Toni},
	title        = {Computational complexity of flat and generic Assumption-Based Argumentation,
	with and without probabilities},
	journal      = {Artif. Intell.},
	volume       = {293},
	pages        = {103449},
	year         = {2021},
	url          = {https://doi.org/10.1016/j.artint.2020.103449},
	doi          = {10.1016/J.ARTINT.2020.103449},
	bibsource    = {dblp computer science bibliography, https://dblp.org}
}

@inproceedings{Baumann11,
	author       = {Ringo Baumann},
	editor       = {James P. Delgrande and
	Wolfgang Faber},
	title        = {Splitting an Argumentation Framework},
	booktitle    = {Logic Programming and Nonmonotonic Reasoning - 11th International
	Conference, {LPNMR} 2011, Vancouver, Canada, May 16-19, 2011. Proceedings},
	series       = {Lecture Notes in Computer Science},
	volume       = {6645},
	pages        = {40--53},
	publisher    = {Springer},
	year         = {2011},
	url          = {https://doi.org/10.1007/978-3-642-20895-9\_6},
	doi          = {10.1007/978-3-642-20895-9\_6},
	bibsource    = {dblp computer science bibliography, https://dblp.org}
}

@inproceedings{BaumannBDW12,
	author       = {Ringo Baumann and
	Gerhard Brewka and
	Wolfgang Dvo\v{r}{\'{a}}k and
	Stefan Woltran},
	editor       = {Esra Erdem and
	Joohyung Lee and
	Yuliya Lierler and
	David Pearce},
	title        = {Parameterized Splitting: {A} Simple Modification-Based Approach},
	booktitle    = {Correct Reasoning - Essays on Logic-Based {AI} in Honour of Vladimir
	Lifschitz},
	series       = {Lecture Notes in Computer Science},
	volume       = {7265},
	pages        = {57--71},
	publisher    = {Springer},
	year         = {2012},
	url          = {https://doi.org/10.1007/978-3-642-30743-0\_5},
	doi          = {10.1007/978-3-642-30743-0\_5},
	bibsource    = {dblp computer science bibliography, https://dblp.org}
}

@article{BondarenkoDKT97,
	author       = {Andrei Bondarenko and
	Phan Minh Dung and
	Robert A. Kowalski and
	Francesca Toni},
	title        = {An Abstract, Argumentation-Theoretic Approach to Default Reasoning},
	journal      = {Artif. Intell.},
	volume       = {93},
	pages        = {63--101},
	year         = {1997},
	url          = {https://doi.org/10.1016/S0004-3702(97)00015-5},
	doi          = {10.1016/S0004-3702(97)00015-5},
	bibsource    = {dblp computer science bibliography, https://dblp.org}
}

@article{Dung95,
	author       = {Phan Minh Dung},
	title        = {On the Acceptability of Arguments and its Fundamental Role in Nonmonotonic
	Reasoning, Logic Programming and n-Person Games},
	journal      = {Artif. Intell.},
	volume       = {77},
	number       = {2},
	pages        = {321--358},
	year         = {1995},
	url          = {https://doi.org/10.1016/0004-3702(94)00041-X},
	doi          = {10.1016/0004-3702(94)00041-X},
	bibsource    = {dblp computer science bibliography, https://dblp.org}
}

@incollection{CyrasFST2018,
	title     = {Assumption-Based Argumentation: Disputes, Explanations, Preferences},
	author    = {Kristijonas Cyras and Xiuyi Fan and Claudia Schulz and Francesca Toni},
	booktitle = {Handbook of Formal Argumentation},
	publisher = {College Publications},
	year      = {2018},
	pages     = {365--408},
	chapter   = {7}
}

@incollection{BrewkaESWW18,
	title     = {Abstract Dialectical Frameworks},
	author    = {Gerhard Brewka and Stefan Ellmauthaler and Hannes Strass and Johannes P. Wallner and Stefan Woltran},
	booktitle = {Handbook of Formal Argumentation},
	publisher = {College Publications},
	Xeditor    = {Pietro Baroni and Dov Gabbay and Massimiliano Giacomin and Leendert van der Torre},
	year      = {2018},
	pages     = {237--285},
	chapter   = {5},
	note = {also appears in IfCoLog Journal of Logics and their Applications 4(8):2263--2318}
}

@string{AI              ={Artificial Intelligence}}

@inproceedings{NielsenP06,
	author       = {S{\o}ren Holbech Nielsen and
	Simon Parsons},
	editor       = {Nicolas Maudet and
	Simon Parsons and
	Iyad Rahwan},
	title        = {A Generalization of Dung's Abstract Framework for Argumentation: Arguing
	with Sets of Attacking Arguments},
	booktitle    = {Argumentation in Multi-Agent Systems, Third International Workshop,
	ArgMAS 2006, Hakodate, Japan, May 8, 2006, Revised Selected and Invited
	Papers},
	series       = {Lecture Notes in Computer Science},
	volume       = {4766},
	pages        = {54--73},
	publisher    = {Springer},
	year         = {2006},
	url          = {https://doi.org/10.1007/978-3-540-75526-5\_4},
	doi          = {10.1007/978-3-540-75526-5\_4},
	bibsource    = {dblp computer science bibliography, https://dblp.org}
}

@article{DvorakKUW24,
	author       = {Wolfgang Dvo\v{r}\'{a}k and
	Matthias K{\"{o}}nig and
	Markus Ulbricht and
	Stefan Woltran},
	title        = {Principles and their Computational Consequences for Argumentation
	Frameworks with Collective Attacks},
	journal      = {J. Artif. Intell. Res.},
	volume       = {79},
	pages        = {69--136},
	year         = {2024},
	Xurl          = {https://doi.org/10.1613/jair.1.14879},
	doi          = {10.1613/JAIR.1.14879},
	bibsource    = {dblp computer science bibliography, https://dblp.org}
}

@inproceedings{Linsbichler14,
	author       = {Thomas Linsbichler},
	editor       = {Simon Parsons and
	Nir Oren and
	Chris Reed and
	Federico Cerutti},
	title        = {Splitting Abstract Dialectical Frameworks},
	booktitle    = {Computational Models of Argument - Proceedings of {COMMA} 2014, Atholl
	Palace Hotel, Scottish Highlands, UK, September 9-12, 2014},
	series       = {Frontiers in Artificial Intelligence and Applications},
	volume       = {266},
	pages        = {357--368},
	publisher    = {{IOS} Press},
	year         = {2014},
	url          = {https://doi.org/10.3233/978-1-61499-436-7-357},
	doi          = {10.3233/978-1-61499-436-7-357},
	bibsource    = {dblp computer science bibliography, https://dblp.org}
}

@article{DvorakZW23,
	author       = {Wolfgang Dvo\v{r}\'{a}k and
	Atefeh {Keshavarzi Zafarghandi} and
	Stefan Woltran},
	title        = {Expressiveness of {SETAF}s and support-free {ADF}s under 3-valued semantics},
	journal      = {J. Appl. Non Class. Logics},
	volume       = {33},
	number       = {3-4},
	pages        = {298--327},
	year         = {2023},
	Xurl          = {https://doi.org/10.1080/11663081.2023.2244361},
	doi          = {10.1080/11663081.2023.2244361},
	bibsource    = {dblp computer science bibliography, https://dblp.org}
}

@inproceedings{KoenigRU22,
	author       = {Matthias K{\"{o}}nig and
	Anna Rapberger and
	Markus Ulbricht},
	editor       = {Francesca Toni and
	Sylwia Polberg and
	Richard Booth and
	Martin Caminada and
	Hiroyuki Kido},
	title        = {Just a Matter of Perspective},
	booktitle    = {Computational Models of Argument - Proceedings of {COMMA} 2022, Cardiff,
	Wales, UK, 14-16 September 2022},
	series       = {Frontiers in Artificial Intelligence and Applications},
	volume       = {353},
	pages        = {212--223},
	publisher    = {{IOS} Press},
	year         = {2022},
	url          = {https://doi.org/10.3233/FAIA220154},
	doi          = {10.3233/FAIA220154},
	bibsource    = {dblp computer science bibliography, https://dblp.org}
}

@inproceedings{LifschitzT94,
	author       = {Vladimir Lifschitz and
	Hudson Turner},
	editor       = {Pascal Van Hentenryck},
	title        = {Splitting a Logic Program},
	booktitle    = {Logic Programming, Proceedings of the Eleventh International Conference
	on Logic Programming, Santa Marherita Ligure, Italy, June 13-18, 1994},
	pages        = {23--37},
	publisher    = {{MIT} Press},
	year         = {1994},
	bibsource    = {dblp computer science bibliography, https://dblp.org}
}

@inproceedings{BaumannBW11,
	author       = {Ringo Baumann and
	Gerhard Brewka and
	Renata Wong},
	editor       = {Sanjay Modgil and
	Nir Oren and
	Francesca Toni},
	title        = {Splitting Argumentation Frameworks: An Empirical Evaluation},
	booktitle    = {Theorie and Applications of Formal Argumentation - First International
	Workshop, {TAFA} 2011. Barcelona, Spain, July 16-17, 2011, Revised
	Selected Papers},
	series       = {Lecture Notes in Computer Science},
	volume       = {7132},
	pages        = {17--31},
	publisher    = {Springer},
	year         = {2011},
	url          = {https://doi.org/10.1007/978-3-642-29184-5\_2},
	doi          = {10.1007/978-3-642-29184-5\_2},
	bibsource    = {dblp computer science bibliography, https://dblp.org}
}

@inproceedings{RussoRT24,
	author       = {Fabrizio Russo and
	Anna Rapberger and
	Francesca Toni},
	editor       = {Pierre Marquis and
	Magdalena Ortiz and
	Maurice Pagnucco},
	title        = {Argumentative Causal Discovery},
	booktitle    = {Proceedings of the 21st International Conference on Principles of
	Knowledge Representation and Reasoning, {KR} 2024, Hanoi, Vietnam.
	November 2-8, 2024},
	year         = {2024},
	url          = {https://doi.org/10.24963/kr.2024/88},
	doi          = {10.24963/KR.2024/88},
	bibsource    = {dblp computer science bibliography, https://dblp.org}
}

@inproceedings{BertholdR024,
	author       = {Matti Berthold and
	Anna Rapberger and
	Markus Ulbricht},
	editor       = {Pierre Marquis and
	Magdalena Ortiz and
	Maurice Pagnucco},
	title        = {Capturing Non-flat Assumption-based Argumentation with Bipolar SETAFs},
	booktitle    = {Proceedings of the 21st International Conference on Principles of
	Knowledge Representation and Reasoning, {KR} 2024, Hanoi, Vietnam.
	November 2-8, 2024},
	year         = {2024},
	url          = {https://doi.org/10.24963/kr.2024/12},
	doi          = {10.24963/KR.2024/12},
	bibsource    = {dblp computer science bibliography, https://dblp.org}
}

@inproceedings{Fan18,
	author       = {Xiuyi Fan},
	editor       = {Tim Miller and
	Nir Oren and
	Yuko Sakurai and
	Itsuki Noda and
	Bastin Tony Roy Savarimuthu and
	Tran Cao Son},
	title        = {On Generating Explainable Plans with Assumption-Based Argumentation},
	booktitle    = {{PRIMA} 2018: Principles and Practice of Multi-Agent Systems - 21st
	International Conference, Tokyo, Japan, October 29 - November 2, 2018,
	Proceedings},
	series       = {Lecture Notes in Computer Science},
	volume       = {11224},
	pages        = {344--361},
	publisher    = {Springer},
	year         = {2018},
	url          = {https://doi.org/10.1007/978-3-030-03098-8\_21},
	doi          = {10.1007/978-3-030-03098-8\_21},
	bibsource    = {dblp computer science bibliography, https://dblp.org}
}

@inproceedings{FanTMW14,
	author       = {Xiuyi Fan and
	Francesca Toni and
	Andrei Mocanu and
	Matthew Williams},
	title        = {Dialogical two-agent decision making with assumption-based argumentation},
	booktitle    = {
	{AAMAS} '14},
	pages        = {533--540},
	year         = {2014},
	url          = {http://dl.acm.org/citation.cfm?id=2615818},
	bibsource    = {dblp computer science bibliography, https://dblp.org}
}

@book{Gabbay:2021,
	publisher = {College Publications},
	volume = {2},
	year = {2021},
	editor = {Dov Gabbay and Massimiliano Giacomin and Guillermo R. Simari and Matthias Thimm},
	title = {Handbook of Formal Argumentation},
	ENTRYTYPE = {book}
}

@inproceedings{Turner96,
	author       = {Hudson Turner},
	editor       = {William J. Clancey and
	Daniel S. Weld},
	title        = {Splitting a Default Theory},
	booktitle    = {Proceedings of the Thirteenth National Conference on Artificial Intelligence
	and Eighth Innovative Applications of Artificial Intelligence Conference,
	{AAAI} 96, {IAAI} 96, Portland, Oregon, USA, August 4-8, 1996, Volume
	1},
	pages        = {645--651},
	publisher    = {{AAAI} Press / The {MIT} Press},
	year         = {1996},
	url          = {http://www.aaai.org/Library/AAAI/1996/aaai96-096.php},
	bibsource    = {dblp computer science bibliography, https://dblp.org}
}

@inproceedings{LehtonenR0W23,
	author       = {Tuomo Lehtonen and
	Anna Rapberger and
	Markus Ulbricht and
	Johannes Peter Wallner},
	editor       = {Pierre Marquis and
	Tran Cao Son and
	Gabriele Kern{-}Isberner},
	title        = {Argumentation Frameworks Induced by Assumption-based Argumentation:
	Relating Size and Complexity},
	booktitle    = {Proceedings of the 20th International Conference on Principles of
	Knowledge Representation and Reasoning, {KR} 2023, Rhodes, Greece,
	September 2-8, 2023},
	pages        = {440--450},
	year         = {2023},
	url          = {https://doi.org/10.24963/kr.2023/43},
	doi          = {10.24963/KR.2023/43},
	bibsource    = {dblp computer science bibliography, https://dblp.org}
}

@inproceedings{BuraglioD0024,
	author       = {Giovanni Buraglio and
	Wolfgang Dvo\v{r}{\'{a}}k and
	Matthias K{\"{o}}nig and
	Markus Ulbricht},
	title        = {Justifying Argument Acceptance with Collective Attacks: Discussions
	and Disputes},
	booktitle    = {Proceedings of the Thirty-Third International Joint Conference on
	Artificial Intelligence, {IJCAI} 2024, Jeju, South Korea, August 3-9,
	2024},
	pages        = {3281--3288},
	publisher    = {ijcai.org},
	year         = {2024},
	url          = {https://www.ijcai.org/proceedings/2024/363},
	bibsource    = {dblp computer science bibliography, https://dblp.org}
}


\cleardoublepage
\appendix

\section{Omitted Proofs}

\SETAFsplittingCF*
\begin{proof}
	(1.) We need to show for each $(T,h)\in R_1\cup R_2\cup R_3$ that $T\cup \{h\}\not\subseteq E=E_1\cup E_2$.
	Let $SF_2'=(A_2',R_2')$ and $SF_2^\star=(A_2^\star,R_2^\star)$.
	If $(T,h)\in R_1$ we immediately get $T\cup \{h\}\not\subseteq E$, since we know $E_1$ is conflict-free in $SF_1$.
	For $(T,h)\in R_2$ there are two cases: either (a)~the attack is removed when we construct the reduct or (b)~the attack remains, i.e., $(T,h)$ in $SF_2^\star$.
	Case (a) happens if some $a\in T\cup\{h\}$ is attacked by $E_1$, i.e., $(T\cup \{h\})\cap (E_1)^+_{R_1\cup R_3}\neq \emptyset$.
	Then at least one argument $a\in T\cup \{h\}$ of the attack does not occur in the modification (i.e., $(T\cup \{h\})\not\subseteq A_2^\star$), and since we assume $E_2\in \cf(SF_2^\star)$ we know $E_2\subseteq A_2^\star$.
	Hence we obtain $T\cup \{h\}\not\subseteq E$.
	For case (b) we get from $E_2\in \cf(SF^\star)$ that at least one argument $a\in T\cup\{h\}$ is not in $E_2$, which also means $T\cup \{h\}\not\subseteq E$.
	Finally, for $(T,h)\in R_3$ we again consider two cases: (a)~$T\cap A_1\subseteq E_1$, and (b) $T\cap A_1 \not\subseteq E_1$.
	For case (a) we either have $T\subseteq A_1$ in which case $h\in (E_1)^+_{R_3}$ and we obtain $h\notin E$ (since then $h\notin A_2'$ while we know $E_2\subseteq A_2'$),
	or if $T\not\subseteq A_1$ we get an attack $(T\cap A_2',h)\in R_2^\star$ (if otherwise $T\cap A_2'=\emptyset$ this means we removed some $a\in T\cap A_2$ when constructing the reduct, which means $a\notin E_2$ and consequently $a\notin E$),
	which since $E_2\in \cf(SF^\star_2)$ either means $T\cap A_2'\not\subseteq E_2$ or $h\notin E_2'$,
	both give us $T\cup \{h\}\not\subseteq E$.
	For case (b) we have $T\cap A_1\not\subseteq E_1$, which means $T\cup \{h\}\not\subseteq E$.

	(2.) Suppose now that $E\in \cf(SF)$. From this we derive that $E \cap A_1\in \cf(SF_1)$ because every subset of a conflict-free set is also conflict-free.
	We now show that $E\cap A_2\in \cf(SF'_2)$.
	Given that $E\in \cf(SF)$, then for all $T\subseteq E\cap A_1$ and $a\in E\cap A_2$, we have $(T,a)\notin R_3$. Hence, no argument in $E$ is deleted going from $SF_2$ to the reduct $SF'_2$. Thus, we conclude that $E\cap A_2\subseteq A'_2$.
	Moreover, by $E\in \cf(SF)$ we know for each $(T,h)\in R_2$ that $T\cup \{h\}\not\subseteq E$ which carries over to $SF_2'$, since the attacks in $R_2$ may be removed, but are never changed.
	Finally, whenever for a link $(T,h)\in R_3$ with $T\cap A_1 \subseteq E$ we
	add an attack $(T\cap A_2',h)\in R_2'$ when constructing the reduct,
	we also obtain $(T\cap A_2')\cup\{h\}\not\subseteq E$
	since otherwise $T\cup\{h\}\subseteq E$.	
	Therefore,  $E\cap A_2\in \cf(SF'_2)$ concluding the proof. 
\end{proof}

\SETAFsplitting*
\begin{proof}
	\textbf{(stable)}. (1.) From Proposition~\ref{prop: cf splitting} together with the assumptions that $E_1\in \stb(SF_1)$ and $E_2\in \stb(SF_2^\star)$, we know that  $E=E_1\cup E_2\in \cf(SF)$.
	Let $a\in A\setminus E$, we show that $a\in E^+_R$.
	If $a\in A_1$ by $E_1\in \stb(SF_1)$ we get $a\in (E_1)^+_{R_1}$, which immediately gives us $a\in E^+_{R}$.
	If $a\in A_2$, either $a\in A'=A^\star$ or $a\notin A'=A^\star$.
	If $a\notin A'=A^\star$ this can only be because $a\in (E_1)^+_{R_3}$ which gives us $a\in E^+_{R}$.
	If $a\in A'=A^\star$ then by $E_2\in \stb(SF^\star)$ we get $a\in (E_2)^+_{R^\star_2}$.
	Hence, either $E_2$ defeats (in $SF^\star_2$) $a$ via some $(T,a)\in R_2$ (in which case $a\in E^+_{R}$) or via some other $(T,a)\in R_2^\star \setminus R_2$, which can only be the remaining part of an attack from $R_3$.
	Clearly in this case $a\notin T$ (as then $(T,a)$ would not defeat $a$), so we know $(T,a)$ is not constructed from an undecided link in $U^{E_1}_{R_3}$ (which do not occur in stable semantics as $A_1=E_1\cup (E_1)^+_{R_1}$). Instead, we must have obtained $(T,a)$ 
	while constructing the reduct, i.e., there is an attack $(T',a)\in R_3$ with $T'\supset T$, and $T'\cap A_1\subseteq E_1$.
	From this we get $T'\subseteq E$, and consequently $a\in E^+_{R_3}$.
	In all cases we get $a\in E^+_R$, which means $A=E\cup E^+$, i.e., $E\in \stb(SF)$.
	
	(2.) Assume $E\in \stb(SF)$. From this we know that $E^\oplus_R=A=A_1\cup A_2$. We first prove that $E_1=E\cap A_1\in \stb(SF_1)$. From Proposition~\ref{prop: cf splitting} we know $E\cap A_1\in \cf(SF_1)$.
	Since $(SF_1,SF_2,R_3)$ is a splitting of $SF$
	we know that the only attacks towards arguments in $A_1$ are from $R_1$, so
	we immediately get $(E\cap A_1)^\oplus_{R_1} =A_1$,
	i.e., $E_1\in \stb(SF_1)$.
	
	We know turn to prove $E_2=E\cap A_2\in \stb(SF_2^\star)$. First, notice that $SF^\star_2=SF'_2$ because $U^{E_1}_{R_3}=\emptyset$.
	From Proposition~\ref{prop: cf splitting} we again obtain
	$E_2\in \cf(SF'_2)$. 
	Let $a\in A_2'\setminus E_2$.
	We show $a\in (E_2)^+_{R_2'}$.
	Since $E\in \stb(SF)$ we know $a\in E^+_R$ which means
	(a)~$a\in E^+_{R_2}$ or (b)~$a\in E^+_{R_3}$.
	In case (a) we have an attack $(T,a)\in R_2$ with $T\subseteq E_2$, and since $E_2\subseteq A_2'$ (which is because $E_2\in \cf(SF')$) and also $a\in A_2'$ by assumption we know $(T,a)\in R_2'$, i.e., $a\in (E_2)^+_{R_2'}$.
	If (b) is the case we know that there is some $(T,a)\in R_3$ with $T\subseteq E$,
	i.e., $T\cap A_1\subseteq E_1$.
	Clearly since $E$ is conflict-free in $SF$ we have $T\cap (E_1)^+_{R_1\cup R_3}=\emptyset$ which means in $SF_2'$ we have an attack $(T\cap A_2',a)$.
	Since $T\cap A_2'\subseteq E_2$ we get $a\in (E_2)^+_{R_2'}$.
	In both cases we get $E_2\cup (E_2)^+_{R_2'}=A_2'$, i.e., $E_2\in \stb(SF'_2)=\stb(SF^\star_2)$.
	
	\textbf{(admissible)}. (1.) Since admissibility implies conflict-freeness, we know from Proposition~\ref{prop: cf splitting} that $E=E_1\cup E_2\in \cf(SF)$.
	We need to show that $E$ defends itself in $SF$, i.e.\ for all $a\in E$, if $(T,a)\in R_1\cup R_2 \cup R_3$, then $(T',t)\in R_1\cup R_2 \cup R_3$ for $T'\subseteq E$ and $t\in T$. 
	Consider an argument $a\in E_1$. $E_1$ defends $a$ from each attack in $R_1$ towards $a$ since $E_1\in \adm(SF_1)$. Therefore, $E_1\in \adm(SF)$. 
	Consider now an argument $a\in E_2$ and an arbitrary attack $(T,a)\in R_2\cup R_3$ towards $a$.
	If $T\cap (E_1)^+_{R_1\cup R_3}\neq \emptyset$ we know $a$ is defended (in $SF$) by $E_1$ against $(T,a)$ and we are done, hence, we proceed with the assumption $T\cap (E_1)^+_{R_1\cup R_3}= \emptyset$.
	This means that either $(T\cap A_2',a)\in R_2^\star$ (via the reduct) or $((T\cap A_2')\cup \{a\},a)\in R_2^\star$ (via the modification).
	Since $a\in E_2$ and $E_2\in \adm(SF_2^\star)$ we know there is a counter-attack in $R_2^\star$ which defends $a$.
	Even in case $((T\cap A_2')\cup \{a\},a)\in R_2^\star$ this counter-attack cannot be against $a$ since this violates conflict-freeness of $E_2$ in $SF_2^\star$.
	Hence, there is some $(S,t)\in R_2^\star$ s.t.\ $S\subseteq E_2$ and $t\in T\cap A_2'$ with $t\notin S$.
	Hence, either (a)~$(S,t)\in R_2$ in which case $a$ is defended by $E$ in $SF$ or (b)~there is some $(S',t)\in R_3$ with $S'\supset S$ s.t.\ $S'\cap A_1\subseteq E_1$, in which case $a$ is defended (in $SF$) by $E$ via the attack $(S',t)$ since then $S'\subseteq E_1\cup E_2$.
	In any case we showed that $a$ is defended in $SF$ by $E$, i.e., $E\in \adm(SF)$.
	
	(2.) By Proposition~\ref{prop: cf splitting} we get $E_1=E\cap A_1\in \cf(SF_1)$ and $E_2=E\cap A_2\in \cf(SF_2')$. Since $E$ is defends itself in $SF$ we get $E\cap A_1\in \adm(SF_1)$ because $(SF_1,SF_2,R_3)$ is a splitting of $SF$, i.e.\ no argument in $E\cap A_1$ is attacked by a subset of $A_2$ or defended by $E\cap A_2$.
	That is, in $SF_1$ every attack towards an argument in $E\cap A_1$ is countered by $E\cap A_1$.
	It remains to show that $E\cap A_2\in \adm(SF^\star_2)$.
	Consider now an argument $a\in E_2$ and an arbitrary attack $(T,a)\in R_2^\star$ against $a$.
	This attack $(T,a)$ either corresponds to an attack $(T,a)\in R_2$ or $(T',a)\in R_3$ with $T'\supset T\setminus \{a\}$ (which accounts for both the case of addition in the reduct and the modification).
	In both cases we have that $T \cap (E_1)^+_{R_1\cup R_3}=\emptyset$ (or $T' \cap (E_1)^+_{R_1\cup R_3}=\emptyset$, resp.) as otherwise $(T,a)$ would not be in $R_2^\star$.
	However, since $a$ is defended by $E$ in $SF$, there is a counter-attack $(S,t)\in R_2\cup R_3$ s.t.\ $S\subseteq E$ and $t\in (T\setminus \{a\})$ (or $t\in (T'\setminus \{a\})$, resp.).
	If $(S,t)\in R_2$ then from $S\subseteq E$ and $E_2\subseteq A_2'$ (which we get from $E_2\in \cf(SF_2')$ via Proposition~\ref{prop: cf splitting}) and the fact that then $(S,t)\in R_2'$ since $S\cup \{t\}\subseteq A_2'$ we get that $E_2$ defends $a$ via $(S,t)$ against $(T,a)$ in $SF_2^\star$.
	If $(S,t)\in R_3$ since $S\subseteq E$ we have $S\cap A_1\subseteq E_1$, and hence we get an attack $(S\cap A_2',t)\in R_2'$ which again defends $a$ against $(T,a)$ in $SF^\star_2$.
	Hence, in every case $a$ is defended in $SF_2^\star$, i.e., $E_2\in \adm(SF_2^\star)$. 
	
	\textbf{(complete)}. (1.) Given statement 1 of admissible semantics proven above, we only need to show that $a\in E_1\cup E_2$ for all $a\in A$ defended by $E_1\cup E_2$ in $SF$. Assume towards contradiction that there is an $a\in (A_1\cup A_2)\setminus (E_1\cup E_2)$ defended by $E_1\cup E_2$. From $E_1\in \com(SF_1)$, we know that $a\notin A_1\setminus E_1$. Hence, $a\in A_2\setminus E_2$ and, because $(SF_1,SF_2,R_3)$ is a splitting and $E_1\cup E_2\in \cf(SF)$, we obtain $a\in A'_2\setminus E_2$. Indeed, if $a\in (E_1)^+_{R_3}$, then $E_1\cup E_2$ defends $a$ from an attack of $E_1$, which is against conflict-freeness of $E_1\cup E_2$. Consider now possible attacks scenarios towards $a$: if $a$ is not attacked, then it would be in every complete extension, hence $E_2\notin \com(SF^\star_2)$. If $a$ is attacked by some set of arguments $T$, then $(T,a)\in R_2$ or  $(T,a)\in R_3$. We show that both cases lead to a contradiction. Consider now $(T,a)\in R_2$. Again, in this case we distinguish three attack scenarios:
	\begin{enumerate}
		\item $(T,a)\in R_2$ with $T\cap (E_1)^+_{R_1\cup R_3}\neq \emptyset$. Since $a$ is defended by $E_1\cup E_2$, such attacks are countered by $E_1$ via a link. Given that $(T,a)\notin R^\star_2$ (eliminated by the reduct), $a$ is vacuously defended by $E_2$ in $SF^\star_2$. Thus, $E_2\notin \com(SF^\star_2)$.
		\item $(T,a)\in R_2\setminus \{(a,a)\}$ with $T\cap (E_1)^+_{R_1\cup R_3}= \emptyset$.
		This means there is a counter-attack $(S,t)\in R_2\cup R_3$ with $S\subseteq E_1\cup E_2$. Given that the reduct and modification do not eliminate such attacks, $a$ is defended by $E_2$ in $SF^\star_2$. Thus, $E_2\notin \com(SF^\star_2)$.
		\item $(T,a)=(a,a)\in R_2$ with $T\cap (E_1)^+_{R_1\cup R_3}= \emptyset$.
		From SETAF Fundamental Lemma~\cite{NielsenP06} and the assumption that $a$ is defended by $E_1\cup E_2$, we get that $E_1\cup E_2\cup \{a\}$ is an admissible (and thus conflict-free) extension in $SF$. Since $a$ is a self-attacking argument, we derive a contradiction.  
	\end{enumerate}
	All of the above derive a contradiction. Therefore, we now consider the case where $(T,a)\in R_3$. Since $E_1\cup E_2$ defends $a$, we know that for some $t\in T\cap A_1$, $(E_1,t)\in R_1$ or for some $t\in T\cap A_2$ and $S\subseteq E_1\cup E_2$, $(S,t)\in R_3$. In the first case, the reduct of $SF_2$ does not contain $(T\cap A'_2,a)$ because $T\cap (E_1)^+_{R_1}\neq \emptyset$. Hence, $a$ is unattacked in $R'_2$. For the same reason, and given that $(E_1)^+_{R_1\cup R_3}\supseteq (E_1)^+_{R_1}$, we also know that  $T\cap (E_1)^+_{R_1\cup R_3}\neq \emptyset$. Thus, $(T,a)\notin U^{E_1}_{R_3}$ and $a$ is unattacked in $R^\star_2$.
	Again, $a$ is vacuously defended by $E_2$ in $SF^\star_2$ and $E_2\notin \com(SF^\star_2)$. Contradiction. Consider now the case where $(S,t)\in R_3$ for some $S\subseteq E_1\cup E_2$ and $t\in T\cap A_2$. if $S\subseteq E_1$, then $T \cap (E_1)^+_{R_3}\neq \emptyset$ and  $(T\cap A_2',a)\notin R_2'$. For the same reason as before, $(T\cap A_2',a)\notin R^\star_2$. Hence, $a$ is vacuously defended by $E_2$ in $SF^\star_2$ and $E_2\notin \com(SF^\star_2)$.
	If $S\not \subseteq E_1$, then $(S\cap A_2',t)\in R_2'$ because $S\cap A_2'\neq \emptyset$, $t\in A_2'$, $S\cap A_1\subseteq E_1$ and $S \cap (E_1)^+_{R_1\cup R_3}= \emptyset$. We derive directly that $(S\cap A_2',t)\in R^\star_2$ since the modification does not delete attacks. Hence, $E_2$ defends $a$ in $SF^\star_2$. Finally, this contradicts our hypothesis that $E_2\in \com(SF^\star_2)$, concluding the proof. 
	
	(2.) Admissibility of $E_1=E\cap A_1$ and $E_2=E\cap A_2$ has been shown above. We need to show that for all $a$ defended by $E_1$ in $SF_1$ and by $E_2$ in $SF^\star_2$, we have $a\in E_1$ and $a\in E_2$ respectively.
	Let us consider $E_1$ first. Towards contradiction, assume there is an $a\in A_1\setminus E_1$ such that $a$ is defended by $E_1$.
	This implies that $a$ is such that $a\in A_1\cup A_2\setminus E$ and $a$ is defended by $E$, in contradiction with the completeness of $E$ in $SF$. 
	Consider now $E_2$. As before, we need to show that there is no $a\in A'_2\setminus E_2$ such that $E_2$ defends $a$ in $SF^\star_2$.
	Again, assume that there is such an $a\in A'_2\setminus E_2$. 
	As before, we consider possible attack scenarios towards $a$. If $a$ does not receive any attack in $SF^\star_2$, then for every attack $(T,a)\in R_2\cup R_3$ it holds that $T\cap (E_1)^+_{R_1\cup R_3}\neq \emptyset$ ($(T,a)$ was eliminated by the reduct). Hence, $E_1$ defends $a$ in $SF$, in contradiction with $E\in \com(SF)$. Assume now that $a$ receives an attack in $SF^\star_2$, then for all $(T,a)\in R^\star_2$ we have $(S,t)\in R^\star_2$ for some $S\subseteq E_2$ and $t\in T$, as we assume that $E_2$ defends $a$ in $R_2^\star$.
	Note that $t\notin S$, as otherwise $(S,t)$ would not counter the attack $(T,a)$ which we assumed.
	But we know $(S,t)\in R^\star_2$ corresponds to some attack in $R_2\cup R_3$.
	If $(S,t)\in R_2$ we have that $E$ defends $a$ in $SF$ (via $(S,t)$), a contradiction to $E\in \com(SF)$.
	If on the other hand there is some $(S',t)\in R_3$ with $S'\supset S$
	we know that also $S'\cap A_1\subseteq E_1$, as otherwise we would have $t\in S$ (if $(S,t)$ was introduced in via the modification), which we already ruled out.
	Finally, the attack $(T,a)$ is either in $R_2$ or corresponds to some attack $(T',a)\in R_3$ with $T'\subseteq T\setminus \{a\}$, in both cases
	$SF$ defends $a$ via $(S,t)$ or $(S',t)$ against the attack.
	Hence, we derive a contradiction to $E\in \com(SF)$.
	As every possible way $a$ could be defended by $E_2$ in $SF_2^\star$ but not in $E_2$ leads to a contradiction, this cannot be the case, hence, $E_2\in \com(SF_2^\star)$.
	
	\textbf{(preferred)}. (1) From statement 1 for admissible semantics above, we derive that $E_1\cup E_2\in \adm(SF)$.
	Moreover, from hypothesis we have that there is no $S_1\in \adm(SF_1)$ such that $S_1\supset E_1$ and no $S_2\in \adm(SF^\star_2)$ such that $S_2\supset E_2$.
	We need to prove that there is no $S\in \adm(SF)$ such that $S\supset E=E_1\cup E_2$.
	Towards contradiction, suppose there is such an $S$. Then $S_1=S\cap A_1\supset E_1$ or  $S_2=S\cap A_2\supset E_2$.
	Consider the first case. Since $E_1\in \prf(SF_1)$ by hypothesis, it must hold that $S_1\notin \adm(SF_1)$. However, this is in contradiction with statement 2 shown above for the admissible semantics (i.e.\ if $S\in \adm(SF)$ and $(SF_1,SF_2,R_3)$ is a splitting for $SF$, then $S\cap A_1 \in \adm(SF_1)$).
	Consider now the case where $S_2\supset E_2$. We can assume $S\cap A_1=E_1$, as otherwise we derive a contradiction as above.
	Similarly to the case before, it must hold that $S_2\notin \adm(SF^\star_2)$. Again, since we assumed $S\in \adm(SF)$, then it must hold that $S_2\in \adm(SF^\star_2)$ (statement 2 of admissible semantics).
	Both directions lead to a contradiction, hence there is no $S\in \adm(SF)$ such that $S\supset E_1\cup E_2$. Thus we conclude that $E_1\cup E_2\in \prf(SF)$. 
	
	(2.) By hypothesis, we have $E\in \prf(SF)$ and hence, there is no $S\in \adm(SF)$ such that $S\supset E$. Moreover, by statement 2 of admissible semantics, we get $E\cap A_1\in \adm(SF_1)$ and $E\cap A_2\in \adm(SF^\star_2)$. Consider now $E\cap A_1$.
	By directionality of preferred semantics~\cite{DvorakKUW24} we obtain $E\cap A_1\in \pref(SF_1)$.
	For $E\cap A_2$, assume now there is an $S_2\in \adm(SF^\star_2)$ such that $S_2\supset E\cap A_2$. For statement 1 of admissible semantics, $(E \cap A_1) \cup S_2$ is admissible in $SF$ which contradicts the maximality of $E$. This conclude the proof. 
	
	\textbf{(grounded)}. (1) Since the grounded extension is also complete, we only need to show that $E_1\cup E_2$ is the minimal complete extension in $SF$. Suppose the contrary is true: there is a set $S\in \com(SF)$ such that $S\subset E_1\cup E_2$. Hence, $S_1=S\cap A_1\subset E_1$ or $S_2=S\cap A_2\subset E_2$. Consider the first case. From the statement 2 of complete semantics above, we derive that $S_1\in \com(SF_1)$. But this contradicts our hypothesis that $E_1\in \grd(SF_1)$, since $S_1$ would be the $\subseteq$-minimal complete extension of $SF_1$. Hence, $S_1=E_1$.
	For the second case, we can assume $S\cap A_1=E_1$ (as otherwise we derive a contradiction via the first case) and we deduce again from statement 2 of complete semantics that $S_2\in \com(SF_2^\star)$.
	However, by hypothesis we have that $E_2= \grd(SF^\star_2)$, which is incompatible with that fact that $S_2\subset E_2$. Both cases lead to a contradiction, hence we derive that $E_1\cup E_2= \grd(SF)$. 
	
	(2) By hypothesis, we have $E\in \grd(SF)$ and hence, there is no $S\in \adm(SF)$ such that $S\subset E$. Moreover, by statement 2 of complete semantics, we get $E\cap A_1\in \com(SF_1)$ and $E\cap A_2\in \com(SF^\star_2)$. Consider now $E\cap A_1$.
	By directionality of grounded semantics~\cite{DvorakKUW24} we obtain $E\cap A_1\in \grd(SF_1)$.
	For $E\cap A_2$, assume now there is an $S_2\in \com(SF^\star_2)$ such that $S_2\subset E\cap A_2$. For statement 1 of complete semantics, $(E \cap A_1) \cup S_2$ is complete in $SF$ which contradicts the minimality of $E$. This conclude the proof. 
\end{proof}

\ABAattacksConserv*
\begin{proof}
	The statement follows from Definition \ref{def: spl aba}, by induction on the depth $k$ of the tree rooted in $\contrary{a}$ with leaves $T\cup \top$. 
	By definition, there is a finite tree rooted in $\contrary{a}$ and leaves $T$ or $\top$.
	For $k=1$, we have one rule $r:\contrary{a}\gets T$ where $T=\{b_1, \dots, b_m\}$. Since $a\in \asm_1$ and $S=atom(S)$, we know $\contrary{a}\in S$. Hence, $T\subseteq S$ and $r\in \rules_1$ by definition of splitting. 
	Assume the statement holds for depth $k$. We show that at depth $k'=k+1$, $R\subseteq \rules_1$ and $T\subseteq \asm_1$.
	At depth $k$ there is at some rules $\{r_1,\dots,r_k\}\subset R$ with $r_i: head(r_i)\gets B_i$ ($1\leq i\leq k$). Again, since $atom(head(r_i))\subseteq S$ and $S=atom(S)$, we get $B_i\subseteq S$ and $r_i\in \rules_1$ for all $i$. Thus $\bigcup_{i=1}^k B_i\subseteq T$, concluding $T\subseteq \asm_1$. 
\end{proof}

\ABAsplittingCF*
\begin{proof}
	For notational convenience, let $E = E_1 \cup E_2$ and let $D'_2 = (\lit_2,\rules'_2,\asm_2,\contraryempty^2)$ be the reduct of $D_2$ w.r.t. $E_1 = E \cap \asm_1$. %
	Moreover, we adapt $S_R^+=\{a \mid S \vdash^R a\}$ and $S_R^\oplus=S \cup S_R^+$ from SETAFs.
	
	(1.) To prove the statement we need to show that there is no $a\in E_1 \cup E_2$ and and $R\in \rules$ such that $E_1\cup E_2\vdash^R \contrary{a}$. 
	Towards contradiction, assume there is indeed such an $a$. 
	Thus either (i) $a\in E_1$ or (ii) $a\in E_2$. 
	Assume (i) is true, that is $\exists a\in E_1$ such that $E_1\cup E_2\vdash^R \contrary{a}$ and $R\in \rules$. From Proposition~\ref{pro:att1-conserv}, we know that $E_2=\emptyset$ and $R\subseteq \rules_1$. Thus, $E_1\vdash^R \contrary{a}$, in contradiction with $E_1 \in \cf(D_1)$.
	Assume now that (ii) is true, i.e.\ $\exists a\in E_2$ and $R\in \rules$ such that $E_1\cup E_2\vdash^R \contrary{a}$. Hence, there is a tree-derivation $\tau$ from $E_1\cup E_2 \cup \{\top\}$ rooted in $\contrary{a}$ and a non-empty set of rules $R_2=R \cap \rules_2$. For each rule $r\in R_2$, there are three possible outcomes when computing $D^{\star}_2$: (a) $r$ does not get removed when computing the reduct; (b) $r$ gets removed and later added in the modification; (c) $r$ gets removed for good. 
	Assume (a) is the case. If a rule $r$ is not removed when computing the reduct, it is modified into a rule $r'\in \rules'_2$ such that $body(r')=body(r)\setminus Th_{D_1}(E_1)$ and $head(r')=head(r)$. Thus, $body(r')$ consists of elements of $E_2$ or atoms derivable from it. Therefore, $E_2\vdash^{\rules^{E_1}_2} \contrary{a}$ and consequently $E_2\vdash^{\rules^{\star}_2} \contrary{a}$ (more rules). Finally, we get $E\notin \cf(D^{\star}_2)$, contradicting our hypothesis.   
	Assume now (b) is the case. By definition of derivation, this means that $E_2$ derives $\contrary{a}$ in $D^{\star}_2$ only if $x_u\in E_2$. However, this contradicts conflict-freeness of $E_2$ in the modification.
	Finally, consider case (c). 
	Since $r$ gets removed, but not added in the modification, we infer that $body(r)\cap \mathsf{IS}_{D_1}(E_1)\neq \emptyset$. 
	Hence, either $\contrary{E_1}\cap body(r)\neq \emptyset$ or $Th_{D_1}((E_1)^+_{\rules_1})\cap body(r)\neq \emptyset$. 
	However, since $E_1\in \cf(D_1)$, this means that either $r$ is a dummy rule or that $\exists b \in body(r)\cap \asm_1\not\subseteq E_1$. Thus, in both cases $E_1\cup E_2\not\vdash^R \contrary{a}$, contradicting our assumption. 
	
	(2.) Suppose now that $E\in \cf(D)$. From this we derive that $E\cap \asm_1 \in \cf(D_1)$ (subset of a conflict-free set). We now show that $E\cap \asm_2\in \cf(D'_2)$. Towards contradiction, assume $E\cap \asm_2\notin \cf(D'_2)$. There is an $a\in E\cap \asm_2$ such that $E\cap \asm_2\vdash^{\rules'_2}\contrary{a}$. By definition of reduct, we know that each $r'\in \rules'_2$ is obtained from a corresponding rule $r\in \rules_2$ such that $body(r)\subseteq body(r')\cup Th_{D_1}(E\cap \asm_1)$. Therefore, $(E\cap \asm_1)\cup (E\cap \asm_2)\vdash^{\rules_1\cup \rules_2} \contrary{a}$. By definition of splitting, we know that $\rules=\rules_1\cup \rules_2$ and $E=(E\cap \asm_1)\cup (E\cap \asm_2)$, deriving $E\vdash^{\rules} \contrary{a}$, and finally $E\notin \cf(D)$. Contradiction.
\end{proof}

We borrow from SETAFs the notions of \emph{projection}, \emph{influence} and the related principle of \emph{Directionality}~\cite{DvorakKUW24}, and apply them to ABA. This will later be used to prove the splitting theorem for ABAFs.

\begin{definition}[projection]
 	Let $D=(\lit,\rules,\asm,\contraryempty)$ be an ABAF with deductive system $(\lit,\rules)$ and $S \subseteq \lit$ a set of sentences such that $S=atom(S)$. We define the projection $D_{\downarrow S}$ of $D$ on $S$ as the ABAF induced by the deductive system $(S,\rules_{\downarrow S})$ with $\rules_{\downarrow S}=\{r\in \rules \mid head(r)\in S, body(r)\subseteq S \}$. Further, $D_{\downarrow S}$ has for assumption set $S\cap \asm$ and contrary function $\contraryempty^S$ the restriction of $\contraryempty$ on $S$ (i.e. $\contrary{a}^S =\contrary{a}$ for all $a\in S$).
\end{definition}

\begin{definition}[influence]
	Let $D=(\lit,\rules,\asm,\contraryempty)$ be an ABAF. An assumption $a\in \asm$ influences $b\in \asm$ if there is a derivation $S\vdash \contrary{b}$ where $a\in S$. 
	Moreover, a set $U \subseteq \asm$ is uninfluenced in $D$ if  no $a\in \asm\setminus U$ influences  any $b\in U$. We  denote the set of uninfluenced sets as $\mathsf{US}(D)$. 
\end{definition}

\begin{principle}[Directionality]
A  semantics $\sigma$ satisfies directionality if  for all ABAFs $D$ and every $U \in \mathsf{US}(D)$ it holds that $$\sigma(D_{\downarrow U})=\{E\cap U \mid E \in \sigma(D)\}.$$
\end{principle}

\begin{observation}\label{pro:ABA Directionality}
	Due to existing results of semantic equivalence between ABAFs and SETAFs, we can trivially infer that for any flat ABAF $D$ and semantics $\sigma\in \{\cf,\adm,\com,\prf,\grd\}$, $\sigma$ satisfies Directionality. Specifically, the notions of projection corresponds to the SETAF counterpart. In fact, for a set $S\subseteq \lit$, we have $SF_{D_{\downarrow S}}=(SF_D)_{\downarrow A}$ where $A$ is the set of arguments corresponding to $\asm\cap S$. Likewise, the notion of influence in ABA is captured by the homonymous SETAF notion since $a\in \asm$ influences $b\in \asm$ when the corresponding argument $a\in A_D$ participates to an attack towards $b\in A_D$.   
\end{observation}

\ABAsplitting*
\begin{proof}
	In what follows, we prove 1. and 2. for each semantics.  

	\textbf{(stable)}. 
	(1.) First notice that since $E_1\in \stb(D_1)$, we have $U^{E_1}_{D_1}=\emptyset$, and consequently $\asm'_2=\asm^{\star}_2$.
	(1.) From Proposition \ref{pro:cf-ABA} together with the hypotheses that $E_1\in \stb(D_1)$ and $E_2\in \stb(D^{\star}_2)$, we know that $E_1\cup E_2\in \cf(D)$. Thus, for any $a\in \asm\setminus E$, we show that $a\in E_R^+$, i.e. $E\vdash^R\contrary{a}$ for some $R\subseteq \rules$. We proceed by cases. Let $a\in \asm_1$. From hypothesis we know that $E_1\vdash^{R_1}\contrary{a}$ for some $R_1\subseteq \rules_1$ which immediately implies $a\in E^+_{R_1}$. Let $a\in \asm_2$. From hypothesis, we know that $E_2\vdash^{R_2}\contrary{a}$ for some $R_2\subseteq\rules'_2$. Thus, for each rule $r'\in \rules'_2$ there is a rule $r\in \rules_2$ such that $body(r)\subseteq body(r)\cup Th_{D_1}(E_1)$. Hence, it follows directly that $E_1\cup E_2=E\vdash^R\contrary{a}$ for some $R\subseteq \rules_1\cup \rules_2=\rules$. 

	(2.) Assume $E \in \stb(D)$. From this we know that $E^\oplus_{\rules} = \asm = \asm_1 \cup \asm_2$. We first prove that $E_1 = E\cap \asm_1 \in \stb(D_1)$. From Proposition \ref{pro:cf-ABA} we know $E\cap \asm_1 \in \cf(D_1)$. Moreover, from Proposition \ref{pro:att1-conserv}, we know that any set of assumptions which is not entirely contained in $\asm_1$ attacks $a\in \asm_1$ via rules in $\rules_1$, therefore we get $E\cap \asm_1\vdash^{R_1}\contrary{a}$ for all $a\in \asm_1\setminus E$ for some $R_1\subseteq\rules_1$. Hence, $E_1\in \stb(D_1)$. 
	We know turn to prove $E_2 = E \cap \asm_2 \in \stb(D'_2)$. We know conflict-freeness holds from Proposition \ref{pro:cf-ABA}. Hence, we only need to show that for every $a\in \asm'_2\setminus E_2$, $E_2\vdash^{R'_2}\contrary{a}$ for some $R'_2 \subseteq \rules'_2$. Since $E\vdash^{R}\contrary{a}$ in $D$, we have two possibilities: (a) $E_1=\emptyset$ or (b) $E_1\neq \emptyset$. If (a) holds, we get $R\subseteq \rules_2$ and $E=E_2\vdash^R\contrary{a}$ where $E_2\subseteq \asm'_2$ and $a\in \asm'_2$. Thus, $E_2\vdash^R\contrary{a}$ holds for some $R\subseteq \rules'_2$. If (b) holds, $E_1\cup E_2\vdash^R \contrary{a}$ in $D$. Therefore, each rule $r\in R\cap \rules_2$ has a corresponding rule $r'\in \rules'_2$ such that $body(r')=body(r)\setminus Th_{D_1}(E_1)$. Since $E_1\cup E_2\in \cf(D)$ by hypothesis, we know that $Th_{D_1}(E_1)\cap \contrary{E_2}=\emptyset$. Hence, $(E\setminus E_1)\vdash^{R'_2}\contrary{a}$ where $R'_2\subseteq \rules'_2$. In both cases we have $E_2\cup (E_2)^+_{\rules'_2}=\asm'_2$, concluding $E_2\in \stb(D'_2)$.

	\textbf{(complete)}. (1.) Given statement 1 of admissible semantics proven above, we only need to show that $a\in E_1\cup E_2$ for all $a\in \asm$ defended by $E_1\cup E_2$ in $D$. Assume towards contradiction that there is an $a\in (\asm_1\cup \asm_2)\setminus (E_1\cup E_2)$ defended by $E_1\cup E_2$. From $E_1\in \com(D_1)$, we know that $a\notin \asm_1\setminus E_1$. Hence, $a\in \asm_2\setminus E_2$. Indeed, if $E_1\vdash \contrary{a}$, then $E_1\cup E_2$ defends $a$ from an attack of $E_1$, which is against conflict-freeness of $E_1\cup E_2$. Consider now possible attacks scenarios towards $a$: if $a$ is not attacked, then it would be in every complete extension, hence $E_2\notin \com(D^\star_2)$. Consider now the case where $a$ is attacked by some set of assumptions $T$, i.e.\ $T\vdash^R \contrary{a}$. 
	If $T\subseteq \asm_2$, by definition of splitting set we have $R\subseteq \rules_2$. 
	In this case we distinguish two attack scenarios: 
	\begin{enumerate}
		\item 
		$T\neq \{a\}$. Since $E_1\cup E_2$ defends $a$, there is an $S\subseteq E_1\cup E_2$ such that $S\vdash^{R'}\contrary{t}$ for some $t\in T$ and $R'\subseteq \rules$. 
		By effect of the reduct, no rule $r\in R'\cap \rules_2$ gets removed. Otherwise, there is a rule $r\in R'\cap \rules_2$ and a sentence $p\in body(r)\cap \lit_1$ such that $p\notin Th_{D_1}(E_1)$. This entails $S\not \vdash^{R'} \contrary{t}$. Thus, for each such $r$ we have $r': head(r) \gets body(r)\setminus Th_{D_1}(E_1)$ in $\rules^{E_1}_2$. Therefore, $S\setminus E_1 \vdash \contrary{t}$ in $D^{E_1}_2$. Moreover, since no rule is removed, no further attacks are added towards $a$ by the modification. We conclude that $E_2$ defends $a$ in $D^{\star}_2$, i.e. $E_2\notin \com(D^{\star}_2)$. Contradiction. 
		\item $T=\{a\}$. From the Fundamental Lemma~\cite{CyrasFST2018} (Thm. 2.13) and the hypothesis that $a$ is defended by $E_1\cup E_2$, we get that $E_1\cup E_2\cup \{a\}$ is an admissible (and thus conflict-free) extension in $D$. Since $a$ is a self-attacking argument, we derive a contradiction.
	\end{enumerate}
	If $T\not \subseteq \asm_2$, since $E_1\cup E_2$ defends $a$, we know that for some $t\in T\cap \asm_1$, $E_1\vdash \contrary{t}$ or for some $t\in T\cap \asm_2$ and $S\subseteq E_1\cup E_2$, $S\vdash^{R'}\contrary{t}$. In the first case, every rule $r\in \rules_2$ whose body contains $t$ (or a sentence derivable from it) gets removed by the reduct. This is guaranteed because $t\notin Th_{D_1}(E_1)$ given that $E_1$ is conflict-free. 
	Therefore, we conclude that $\contrary{a}$ is not derivable in $D^{E_1}_2$ ($a$ is unattacked). Further, each rule $r$ removed under the reduct is not reintroduced via the modification, since $t\in (E_1)^+_{\rules_1}$ from hypothesis, and therefore $body(r)\cap \mathsf{IS}_{D_1}(E_1)\neq \emptyset$. As a result,  $\contrary{a}$ is not derivable in $D^{\star}_2$, meaning that $a$ is vacuously defended by $E_2$, making $E_2$ not complete. Contradiction. Consider now the second case. It can be easily derived that $E_2$ defends $a$ in $D^{\star}_2$ through similar considerations case (1) above. 
	
	(2.) $E_1=E\cap A_1$ and $E_2=E\cap A_2$ has been shown to be admissible sets above. We need to show that for all $a$ defended by $E_1$ in $D_1$ and by $E_2$ in $D^\star_2$, we have $a\in E_1$ and $a\in E_2$ respectively.
	Let us consider $E_1$ first. Towards contradiction, assume there is an $a\in \asm_1\setminus E_1$ such that $a$ is defended by $E_1$.  
	This directly implies that $a\in \asm_1\cup \asm_2\setminus E$ and $a$ is defended by $E$, in contradiction with the completeness of $E$ in $D$. 
	Consider now $E_2$. As before, we need to show that there is no $a\in \asm_2\setminus E_2$ such that $E_2$ defends $a$ in $D^\star_2$.
	Again, assume that there is such an $a\in \asm_2\setminus E_2$. 
	As before, we consider possible attack scenarios towards $a$. If $a$ does not receive any attack in $D^\star_2$, then for every $T\subseteq \asm$ and $R\subseteq \rules_2$ such that $T\vdash^R \contrary{a}$ in $D$, then for some $r\in R$ it holds that $body(r) \cap \mathsf{IS}_{D_1}(E_1)\neq \emptyset$ ($r$ was eliminated by the reduct and not reintroduced). Hence, $E_1$ defends $a$ in $D$, in contradiction with $E\in \com(D)$. 
	Assume now that $a$ receives an attack in $D^\star_2$, then for all $T\subseteq \asm_2$ and $R\in \rules^\star_2$ such that $T\vdash^R \contrary{a}$, there is an $S\subseteq E_2$ and $R'\in \rules^\star_2$ such that $S\vdash^{R'} \contrary{t}$ with $t\in T$ ($E_2$ defends $a$ in $D_2^\star$). 
	Note that $\{t,x_u\}\not \subseteq S$, as $E_2\in \cf(D_2^\star)$. 
	But we know $S\vdash^{R'} \contrary{t}$ corresponds to some derivation $S'\vdash^{R''} \contrary{t}$ with $R''\subseteq \rules$ and $S'\supseteq S$. If $S'=S$, we know that $E$ defends $a$ in $D$, contradicting the hypothesis that $E\in \com(D)$. 
	If on the other hand $S'\supset S$, we know that also $S'\cap \asm_1\subseteq E_1$, as otherwise some rule $r\in R''$ would be removed by the reduct. 
	Finally, for the attack $T\vdash^R \contrary{a}$, we have two distinct possibilities: either $R\subseteq \rules_2$ (the derivation is unaffected by reduct and modification) or it corresponds to some derivation $T'\vdash \contrary{a}$ with $T'\cap \asm_1\neq \emptyset$ and $T'\subseteq T\setminus \{x_u\}$. In both cases $D$ defends $a$ via $S$ or $S'$ against the attack. 
	Hence, we derive a contradiction to $E\in \com(D)$.
	Every possible way in which $a$ could be defended by $E_2$ in $D_2^\star$, but not in $E_2$, leads to a contradiction. Hence, $E_2\in \com(D_2^\star)$. 
	
	The proof for preferred and grounded semantics is constructed in the same way as the one for SETAFs, by replacing arguments with assumptions and appealing to Observation~\ref{pro:ABA Directionality} above. In particular, observe that for the splitting set $S$, we have $D_{\downarrow S}=D_1$. 
\end{proof}

\end{document}